\documentclass[11pt,reqno,titlepage]{amsart}
\usepackage{tikz}
\usepackage{subfig}
\usepackage{epstopdf}
\usepackage{epsfig}
\usepackage{lipsum}
\usepackage{hyperref}
\usepackage{amsthm}
\usepackage{url}
\usepackage{amsmath}
\usepackage{amsfonts}
\usepackage{lastpage}
\usepackage{graphicx}
\usepackage{multirow}
\usepackage{tabu}
\usepackage{essay-def}
\usepackage{clrscode}
\usepackage{caption}
\usepackage{appendix}
\usepackage{mathrsfs}
\usepackage{algorithm}
\usepackage{algorithmic}
\usepackage{natbib}
\usepackage{tikz-cd}
\usepackage{stackengine}
\setstackEOL{\\}
\usetikzlibrary{arrows,decorations.pathmorphing,backgrounds,positioning,fit,petri,calc,shapes.misc,decorations.markings}

\usepackage[foot]{amsaddr}

\usepackage{tikz}
\usetikzlibrary{shapes,arrows}

\tikzstyle{decision} = [diamond, draw, fill=blue!20, 
    text width=4.5em, text badly centered, node distance=3cm, inner sep=0pt]
\tikzstyle{block} = [rectangle, draw, fill=blue!20, 
    text width=5em, text centered, rounded corners, minimum height=4em]
\tikzstyle{line} = [draw, -latex']
\tikzstyle{cloud} = [draw, ellipse,fill=red!20, node distance=3cm,
    minimum height=2em]

\usetikzlibrary{positioning}
\tikzset{main node/.style={circle,fill=blue!20,draw,minimum size=1cm,inner sep=0pt},  }

\makeatletter
\g@addto@macro{\endabstract}{\@setabstract}
\newcommand{\authorfootnotes}{\renewcommand\thefootnote{\@fnsymbol\c@footnote}}%
\makeatother

\newcommand{\titlemark}{Projected Wasserstein gradient descent for high-dimensional Bayesian inference}
\title{\textbf{\titlemark}}

\begin{document}
\begin{center}
  \LARGE 
  \titlemark
  \par \bigskip

  \normalsize
  \authorfootnotes
  Yifei Wang\footnote{\texttt{wangyf18@stanford.edu}}\textsuperscript{1},
  Peng Chen \footnote{\texttt{peng@oden.utexas.edu}}\textsuperscript{2} and 
  Wuchen Li\footnote{\texttt{wuchen@mailbox.sc.edu}}\textsuperscript{3}
 \par \bigskip
 \textsuperscript{1}Department of Electrical Engineering, Stanford University \par
  \textsuperscript{2} Oden Institute for Computational Engineering and Sciences, The University of Texas at Austin\par
   \textsuperscript{3}Department of Mathematics, University of South Carolina\par \bigskip
\end{center}

\begin{abstract}
We propose a projected Wasserstein gradient descent method (pWGD) for high-dimensional Bayesian inference problems. The underlying density function of a particle system of WGD is approximated by kernel density estimation (KDE), which faces the long-standing curse of dimensionality. We overcome this challenge by exploiting the intrinsic low-rank structure in the difference between the posterior and prior distributions. The parameters are projected into a low-dimensional subspace to alleviate the approximation error of KDE in high dimensions. We formulate a projected Wasserstein gradient flow and analyze its convergence property under mild assumptions. Several numerical experiments illustrate the accuracy, convergence, and complexity scalability of pWGD with respect to parameter dimension, sample size, and processor cores.
\end{abstract}

\section{Introduction}
Bayesian inference is of central importance for inverse problems across many fields in machine learning, uncertainty quantification, data assimilation for computational science and engineering.
It provides a powerful tool and principled framework in quantifying uncertainty of a given system by fusing complex data and models. Given observations of a system with uncertain parameters, Bayesian inference characterizes the optimal distribution of parameters by drawing samples from the posterior distribution. The key task is to efficiently draw samples from the posterior distribution, especially for high-dimensional parameter spaces.

Classical approaches for Bayesian inference include Markov chain Monte Carlo (MCMC) based methods and transport-based variational inference. 
 Many MCMC methods \citep{acc_diff, hmcmc, blvsg, sgmcmc, nrsss, itaao, ildgs} simulate a continuous-time diffusion process which keeps the posterior distribution as the invariant distribution. It is known that the Kolmogorov forward generator of overdamped Langevin dynamics in MCMC is equivalent to the gradient descent direction of the Kullback-Leibler (KL) divergence in probability space with the Wasserstein metric \citep{jko, gfims}. In this view point, these MCMC-related diffusion processes can also be viewed as modified Wasserstein gradient flows. 

Closely related to this perspective, transport-based variational inference methods \citep{SVGD, Liu2017_steina, asvnm, otgos} 
seek to minimize the KL divergence between the transported distribution and the posterior distribution. They can also be viewed as gradient descent methods or Newton's methods in generalized Wasserstein space. Similarly, the Wasserstein Newton's method has been discussed in \citep{infs}. 

However, all of these sampling methods may face the curse of dimensionality for solving high-dimensional Bayesian inference problems. Namely, to achieve a certain level of accuracy, the computational complexity of Wasserstein gradient direction of Boltzmann-Shannon entropy grows rapidly (usually exponential) with respect to the parameter dimension. To alleviate the curse of dimensionality, different approaches have been developed by exploiting the sparsity \citep{MarzoukNajmRahn07a, SchwabStuart12,SchillingsSchwab13,ChenSchwab15, ChenSchwab16,ZechMarzouk20} or intrinsic low dimensionality of the posterior distribution with respect to the parameter \cite{Bui-ThanhBursteddeGhattasEtAl12_gbfinalist, Bui-ThanhGhattasMartinEtAl13,CuiLawMarzouk16,ChenVillaGhattas17,BeskosGirolamiLanEtAl17,ZahmCuiLawEtAl18,BrennanBigoniZahmEtAl20}. 
In fact, the intrinsic low dimensionality, as revealed by derivative information with respect to the parameter, e.g., gradient, Hessian, or higher-order derivatives, has been observed in many uncertainty quantification problems such as model reduction for sampling and deep learning \cite{BashirWillcoxGhattasEtAl08,ChenGhattas19a, AlgerChenGhattas20, OLeary-RoseberryVillaChenEtAl20}, optimization under uncertainty \cite{AlexanderianPetraStadlerEtAl17,ChenVillaGhattas19, ChenHabermanGhattas21}, and Bayesian optimal experimental design \citep{AlexanderianPetraStadlerEtAl16,CrestelAlexanderianStadlerEtAl17,WuChenGhattas20}. 
In particular, the low dimensionality detected by gradient and Hessian information has been used in projected Stein variational gradient descent \citep{ChenGhattas20} and Newton \citep{ChenWuChenEtAl19} methods, respectively, to solve high-dimensional Bayesian inference problems. The complexity of these methods grows slowly or remains the same with increasing parameter dimension.

\textbf{Contributions:}
In this paper, we present a projected version of Wasserstein gradient descent method (WGD) to sample from high-dimensional posterior distributions. It integrates the merits of the fast convergence of Wasserstein gradient flow with theoretical guarantee and favorable scaling to high dimensions with certified projection errors. Specifically, in implementing the WGD we use a kernel density estimation (KDE) as an approximation of a log density term---the gradient of entropy. However, it is well known that KDE suffers from the curse of dimensionality, i.e., the approximation accuracy quickly deteriorates with respect to the parameter dimension, which renders WGD samples an inaccurate representation of the posterior in high dimensions. To address this challenge, we project the parameters to a low-dimensional subspace, which is constructed by the gradient information of the log-likelihood. Then the KDE is only built in the low-dimensional subspace. To further reduce the subspace dimension, we propose a batched version of KDE, which allows us to use KDE in a batched subspace of (arbitrarily) small dimensions. Based on the projection of the parameters and KDE, we develop a projected WGD (pWGD) and formulate a corresponding Wasserstein gradient flow (WGF) in the subspace. Under suitable assumptions, we show that this projected WGF can be embedded in the full space as a special WGF, which we further prove to converge as fast as (if not faster than) the original WGF without projection. Furthermore, thanks to the parallel property of the ensemble transport of the samples, we also develop a parallel implementation of pWGD that distribute samples to multiple processors. 
Finally we demonstrate the (higher) accuracy and (faster) convergence compared to WGD, SVGD, pSVGD, as well as the scalability of computational complexity with respect to parameter dimension, sample size, and processor cores. A variety of numerical experiments are conducted with both low and high-dimensional parameters, both synthetic and real-world data, both toy model and simplified differential equation models used in environmental engineering, geoscience, and epidemics for COVID-19.

The rest of the paper is organized as follows. In section \ref{sec:lws}, we introduce three major sampling methods, Langevin MCMC, WGD and SVGD with analysis on their differences and relations. The proposed pWGD algorithm is presented and analyzed through the lens of WGF in Section \ref{sec:pwgd}. Numerical experiments are presented in section \ref{sec:num}. 

\section{Langevin MCMC, WGD and SVGD}\label{sec:lws}
Let $x\in \mbR^d$ denote a random parameter with prior density $p_0:\mbR^d\to\mbR$. Let $y=\{y_i\}_{i=1}^m$ denote a set of i.i.d. observation data. We denote 
$$
f(x)=\prod_{i=1}^mp(y_i|x),
$$
as a continuous likelihood of $y$ at point $x$. Then, the posterior density of $x$ conditioned on data $y$ is given by the Bayes' rule
$$
\pi(x)=\frac{1}{Z} p_0(x)f(x),
$$
where $Z$ is a normalization constant defined as
$$
Z=\int p_0(x)f(x)dx,
$$
which is often intractable to compute. The central task of Bayesian inference is to draw samples following the posterior distribution with density $\pi(x)$. Then the samples can be used to compute some statistical quantities, e.g., mean, variance, quantiles, to facilitate system optimization and decision making under uncertainty.

Langevin MCMC is a classical approach for sampling from Bayesian posterior. Briefly, to sample from a target distribution with density $\pi$, we evolve a particle system $\{X_l^n\}_{n=1}^N$ according to the update rule
\begin{equation}\label{equ:lmcmc}
x_{l+1}^n = x_l^n+\alpha_l \nabla \log \pi(x_l^n) -\sqrt{2\alpha }Z_l^n,
\end{equation}
where $Z_l^n\sim \mcN(0,I_d)$ follows the standard Gaussian distribution and $\alpha_l>0$ is the step size. 

From another perspective, to sample from a target density, we can also consider a deterministic update rule by Wasserstein gradient descent (WGD):
\begin{equation}\label{equ:wgd}
x_{l+1}^n =x_l^n+\alpha_l(\nabla \log \pi(x_l^n)-\xi_l(x_l^n)),
\end{equation}
where $\xi_l:\mbR^d\to \mbR^d$ is a vector-valued function which approximates $\nabla \log \rho_l$. Here we assume that $\{x_l^n\}_{n=1}^N$ follows an underlying distribution $\rho_l$. Compared to Langevin MCMC, WGD has a deterministic update rule with particle interactions and the randomness only comes from the initial positions of the particle system. 

In this paper, we use the kernel density estimation (KDE) $\tilde \rho_l(x)=\sum_{n=1}^N k(x,x^n_l)$ to approximate $\rho_l(x)$. Here $k(x,x'):\mbR^d\times \mbR^d\to \mbR$ is a given positive kernel function. Using KDE as an approximation of $\rho_l$, the mapping $\xi_l(x)$ is given by
$$
\xi_l(x)=\nabla \log \tilde \rho_l(x)=\frac{\sum_{n=1}^N\nabla_x k(x,x_l^n)}{\sum_{n=1}^Nk(x,x_l^n)}.
$$
One key drawback of KDE is that its approximation of the gradient of the log density term $\nabla \log \rho_l$ is often significantly deteriorated in high dimensions, known as curse of dimensionality, see \cite{fomde, nkdea}.

Another deterministic update rule was proposed by \citet{SVGD} as a variational inference method, which is called Stein variational gradient descent (SVGD). In each iteration, it updates the particles as 
\begin{equation}\label{equ:svgd}
    x_{l+1}^n = x^n_l+\alpha_l \hat \phi_l(x_l^n),
\end{equation}
where $\hat \phi_l:\mbR^d\to \mbR^d$ is given by
$$
\hat \phi_l(x) = \frac{1}{N}\sum_{n=1}^N\pp{\nabla_{x_l^m}\log \pi(x_l^m) k(x_l^m,x)+\nabla_{x_l^m} k(x_l^m,x)}. 
$$
Here $k(x,y)$ is a positive kernel function.

\subsection{Relation}
We present a continuous-time dynamical viewpoint to illustrate that the particle system $\{X_l^n\}_{n=1}^N$ generated by Langevin MCMC \eqref{equ:lmcmc}, WGD \eqref{equ:wgd}, and SVGD \eqref{equ:svgd} approximate the posterior distribution.

With $\alpha_l\to 0$, the Langevin MCMC \eqref{equ:lmcmc} can be viewed as the time discretization of the overdamped Langevin dynamic:
\begin{equation}\label{equ:old}
dx_t = \nabla \log \pi(x_t) +\sqrt{2}dW_t,
\end{equation}
where $W_t$ is the standard Brownian motion. 
The probability density $\rho_t$ of the overdamped Langevin dynamics \eqref{equ:old} satisfies the Fokker-Planck equation
\begin{equation}\label{equ:fp}
\p_t\rho_t = -\nabla\cdot(\rho_t\nabla \log \pi)+\Delta \rho_t.
\end{equation}

In continuous time, WGD approximates the following deterministic mean-field particle system, namely mean-field Wasserstein dynamics:
\begin{equation}\label{equ:mfwd}
dx_t = \pp{ \nabla \log \pi(x_t)-\nabla\log \rho_t(x_t)}dt.
\end{equation}
Here the terminology mean-field implies that the evolution of dynamics \eqref{equ:mfwd} depends on the current density function $\rho_t$. The Liouville equation of the dynamic system \eqref{equ:mfwd} writes
\begin{equation}\label{equ:lio}
\p_t\rho_t = -\nabla\cdot(\rho_t(\nabla \log \pi-\nabla\log \rho_t)).
\end{equation}
We note that equations \eqref{equ:fp} and \eqref{equ:lio} are identical by
\begin{equation*}
\begin{aligned}
\p_t\rho_t = &-\nabla\cdot(\rho_t(\nabla \log \pi-\nabla \log \rho_t))\\
=&-\nabla\cdot(\rho_t\nabla \log \pi)+\Delta \rho_t.
\end{aligned}
\end{equation*}
This is because $\nabla\cdot(\rho_t\nabla\log\rho_t)=\nabla\cdot(\nabla \rho_t)=\Delta\rho_t$.

From an optimization perspective, they all correspond to the gradient flow of the KL divergence under the Wasserstein metric. To be concrete, consider the following optimization problem in the probability space
\begin{equation}\label{prob}
\inf_{\rho \in \mcP(\Omega)} \mathrm{D}_\mathrm{KL}(\rho\|\pi)=\int (\log \rho -\log \pi )\rho dx.
\end{equation}
Here $\mcP(\Omega)=\{\rho\in \mcF(\Omega)|\int \rho dx=1,\rho \geq 0\}$, $\mcF(\Omega)$ denotes the set of all smooth functions on $\Omega$ and $\mathrm{D}_\mathrm{KL}(\rho\|\pi)$ evaluates the KL divergence from $\rho$ to $\pi$. The Wasserstein gradient flow of $\mathrm{D}_\mathrm{KL}(\rho\|\pi)$ writes
\begin{equation}\label{equ:wgf}
\p_t\rho_t= \nabla\cdot\pp{\rho_t\nabla \frac{\delta}{\delta \rho_t} \mathrm{D}_\mathrm{KL}(\rho_t\|\pi)}.
\end{equation}
Here $\frac{\delta }{\delta \rho}$ represents the $L^2$ first-order variation operator. A brief introduction to metrics and gradient flows in the probability space can be found in 
\citep{gfims}. 

By taking the KL divergence into the Wasserstein gradient flow \eqref{equ:wgf}, we immediately recover the Fokker-Planck equation \eqref{equ:fp} and the Liouville equation \eqref{equ:lio}. As a result, with $t\to \infty$, $\rho_t$ following the gradient flow \eqref{equ:wgf} asymptotically converges to the unique minimizer $\pi$ of \eqref{prob}. 



On the other hand, SVGD approximates the following particle dynamics, called mean-field Stein dynamics:
\begin{equation}\label{dyn:svgd}
dx_t = \phi_t(x_t)dt,    
\end{equation}
where $\phi_t:\mbR^d\to \mbR^d$ is a transport mapping given by
$$
\phi_t(x) = \int \pp{\nabla_y\log\pi(y)k(y,x)+\nabla_y k(y,x)}\rho_t(y)dy.
$$
Indeed, $\phi_t$ approximates $\nabla \log \pi-\nabla \log \rho_t$ in the reproducing kernel Hilbert space (RKHS) \citep[Theorem 2]{uaapb}.  
The Liouville equation of the dynamics \eqref{dyn:svgd} writes:
\begin{equation}
    \p_t\rho_t = -\nabla\cdot(\rho_t \phi_t).
\end{equation}
In short, the above equation is an approximation of the Liouville equation \eqref{equ:lio} in RKHS. 


The connections among Langevin MCMC, WGD and SVGD are explained in Figure \ref{fig:rel}.

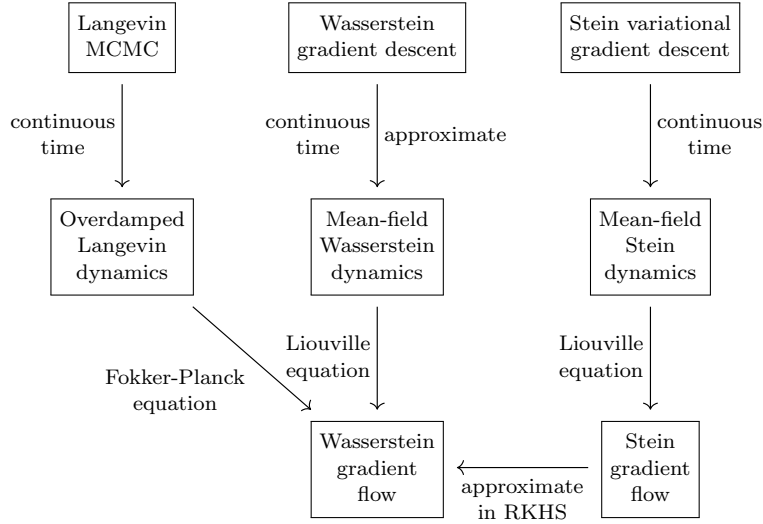
\begin{figure}[ht]
\centering
 \begin{tikzcd}[row sep=huge]{\scriptsize
\fbox{\Centerstack[c]{ Langevin \\ MCMC }}} 
\arrow[d,"{\scriptsize\setlength{\fboxsep}{0pt}
  \setlength{\fboxrule}{0pt}\fbox{\Centerstack[c]{continuous\\time}}}"']
& 
{\scriptsize\fbox{\Centerstack[c]{Wasserstein\\ gradient descent}}}
\arrow[d,"{\scriptsize\setlength{\fboxsep}{0pt}
  \setlength{\fboxrule}{0pt}\fbox{\Centerstack[c]{continuous\\time}}}"',"{\scriptsize\setlength{\fboxsep}{0pt}
  \setlength{\fboxrule}{0pt}\fbox{\Centerstack[c]{approximate}}}"]
&{\scriptsize\fbox{\Centerstack[c]{Stein variational\\gradient descent}}}
\arrow[d,"{\scriptsize\setlength{\fboxsep}{0pt}
  \setlength{\fboxrule}{0pt}\fbox{\Centerstack[c]{continuous\\time}}}"]  \\
{\scriptsize\fbox{\Centerstack[c]{Overdamped\\ Langevin \\dynamics}}} \arrow[rd,"{\scriptsize\setlength{\fboxsep}{0pt}
  \setlength{\fboxrule}{0pt}\fbox{\Centerstack[c]{Fokker-Planck \\ equation}}}"'] 
& {\scriptsize\fbox{\Centerstack[c]{Mean-field\\Wasserstein \\ dynamics}}} 
\arrow[d,"{\scriptsize\setlength{\fboxsep}{0pt}
  \setlength{\fboxrule}{0pt}\fbox{\Centerstack[c]{Liouville \\ equation}}}"'] 
& {\scriptsize\fbox{\Centerstack[c]{Mean-field\\Stein \\ dynamics}}} \arrow[d,"{\scriptsize\setlength{\fboxsep}{0pt}
  \setlength{\fboxrule}{0pt}\fbox{\Centerstack[c]{Liouville \\ equation}}}"']   \\
& {\scriptsize{\fbox{\Centerstack[c]{Wasserstein \\gradient \\ flow}}}}
& {\scriptsize{\fbox{\Centerstack[c]{Stein \\gradient \\ flow}}}} \arrow[l,"{\scriptsize\setlength{\fboxsep}{0pt}\setlength{\fboxrule}{0pt}\fbox{\Centerstack[c]{approximate \\ in RKHS}}}"]
  \end{tikzcd}
\caption{\small Illustration of connections for Langevin MCMC, WGD, and SVGD and their corresponding dynamics and continuous flow.}\label{fig:rel}
\end{figure}


\section{Projected Wasserstein gradient descent}\label{sec:pwgd}
As previously mentioned, the high dimensionality $d$ may impede the approximation of $\nabla \log \rho_l$ in WGD. To tackle this curse of dimensionality, we employ a projection method for dimension reduction to evolve particles in a $r$-dimensional subspace of $\mbR^d$. Here $r$ is usually significantly smaller than $d$. To find such a data-informed parameter subspace, we use the gradient information of the log-likelihood by defining a matrix $H\in \mbR^{d\times d}$ as
\begin{equation}\label{equ:h}
    H = \int \nabla \log f (\nabla \log f)^T \pi dx.
\end{equation}
Suppose that the prior density $p_0$ is given in the general form of
\begin{equation}\label{prior_stru}
p_0\propto \exp(-V-W),    
\end{equation} 
where $\nabla^2V(x)\succeq \Gamma$ for a positive semi-definite matrix $\Gamma$ and $W$ is a bounded function. For instance, $\Gamma$ can be the covariance of a Gaussian prior. Let $(\lambda_i,\psi_i)_{i=1}^r$ denote the dominant generalized eigen-pair of $(H,\Gamma)$, where $(\lambda_i,\psi_i)$ corresponds to the $i$-th largest generalized eigenvalue
\begin{equation}\label{eq:genEignen}
H\psi_i=\lambda_i\Gamma\psi_i.    
\end{equation}

An important observation is that the eigenvalue $\lambda_i$ evaluates the sensitivity of the variation of likelihood w.r.t. parameters along the direction $\psi_i$. In other words, the likelihood function $f$ does not change much in directions $\psi_i$ with small eigenvalue $\lambda_i$. In practice, we truncate at $r$ such that $\lambda_r \geq \varepsilon \geq \lambda_{r+1}$ for a given tolerance $\varepsilon$. Moreover, we apply a randomized SVD algorithm \citep{SaibabaLeeKitanidis16} to compute $(\lambda_i,\psi_i)_{i=1}^r$, which takes $O(r)$ matrix-vector products in computation cost.

Define a linear projection operator:
$$
P_r x = \sum_{i=1}^r \lra{\psi_i,x}\psi_i=\Psi_r w.
$$
Here we denote $\Psi_r = (\psi_1,\dots,\psi_r)\in \mbR^{d\times r}$ and $w\in \mbR^r$ with entries $w_i=\lra{\psi_i,x}$. Based on this projection operator, we seek a profile function $g:\mbR^r\to \mbR^m$ such that $g(P_rx)$ is a good approximation of the likelihood $f(x)$. For a given profile function $g$, we define a projected density $\pi_r:\mbR^d\to\mbR$ as follows
$$
\pi_r(x) := \frac{1}{Z_r} g(P_r x)p_0(x),
$$
where $Z_r$ is a normalization constant defined by $Z_r=\int g(P_r x)p_0(x)dx$. According to \citep{ZahmCuiLawEtAl18}, there exists an optimal profile function $g^*$ such that
$$
\mathrm{D}_\mathrm{KL}(\pi |\pi_r^*)\leq \mathrm{D}_\mathrm{KL}(\pi|\pi_r).
$$
Here $\pi_r^*$ is the projected density defined by the optimal profile function $g^*$. This profile function can be explicitly given by 
$$
g^*(x) =\int_{X^\perp} f(P_rx+z)p_0^\perp(z|P_rx)dz.
$$
Here $X^\perp$ is the complement of the subspace $X_r$ spanned by $\psi_1,\dots,\psi_r$ and the density function $p_0^\perp(z|P_rx)$ is defined by
$$
p_0^\perp (z|P_rx):=p_0(z+P_rx)/p_0^r(P_rx),
$$
where $p_0^r(P_rx):=\int_{X^\perp} p_0(P_rx+z)dz$.

Moreover, under the assumption that the prior satisfies \eqref{prior_stru}, it is shown in \citep{ZahmCuiLawEtAl18} that 
\begin{equation}\label{kl_error}
    \mathrm{D}_\mathrm{KL}(\pi |\pi _r^*)\leq\frac{\gamma}{2}\sum_{i=r+1}^d\lambda_i.
\end{equation}
Here the constant is $\gamma=\exp(\sup W(x)-\inf W(x))$. 

By the projection operator $P_r$, we can decompose the prior for the parameter $x=x^r+x^\perp$ (with $x^r := P_r x = \Psi_r w$) as
$$
p_0(x)=p_0^r(x^r)p_0^\perp(x^\perp|x^r).
$$
Because $p_0^r$ only depends on $x^r=\Psi_rw$, a prior density for $w$ can be defined as
$$
\tilde p_0(w)=p_0^r(\Psi_rw). 
$$
Based on this prior for $w$ and the optimal profile function $g=g^*$, the posterior for $w$ writes:
$$
\tilde \pi(w) = \frac{1}{\tilde Z}g(\Psi_rw)\tilde p_0(w), 
$$
where $\tilde Z=\int g(\Psi_r w)\tilde p_0(w)dw$ is a normalization constant. With $\tilde \pi(w)$, we can rewrite the projected density $\pi_r$ by
$$
\pi_r(x) = \tilde \pi(w)p_0^\perp(x^\perp|\Psi_rw).
$$
Hence, to sample $x$ from $\pi_r$, it suffices to sample $w$ from $\tilde \pi$ and $x^\perp$ from $p_0^\perp(x^\perp|\Psi_rw)$. 

To sample from $\tilde \pi$, we apply WGD for $w$ as
\begin{equation}\label{equ:pwgd}
    w^n_{l+1} = w^n_l+\alpha_l^n(\nabla_w \log \tilde \pi(w^n_l)-\xi^r(w^n_l)). 
\end{equation}
Here $\xi^r$ is an approximation of $\nabla \log \rho^{r}_l(w)$, where we assume that $\{w^n_l\}_{n=1}^N$ follow the distribution of $\rho^{r}_l$. Similarly, we can use KDE in calculating the approximation
$$
\xi^r(w)=\frac{\sum_{n=1}^N\nabla k^r(w,w^n_l)}{\sum_{n=1}^Nk^r(w,w^n_l)}.
$$
Here $k^r:\mbR^r\times \mbR^r\to \mbR$ is a positive kernel function. A common choice is a Gaussian kernel
$$
k^r(w,w')=\exp\pp{-\frac{\|w-w'\|_2^2}{2h}},
$$
where $h$ is computed at the current samples $w_l^n$, e.g., as the median of their square distances \citep{SVGD} or through optimization \citep{aigf}.
For the step size $\alpha_l^n$ in \eqref{equ:pwgd}, we use a line search technique \citep{ChenWuChenEtAl19, ChenGhattas20}.

In practice, to construct the basis $\Psi_r$, instead of computing $H$ in \eqref{equ:h} with samples from posterior, which are not available, we approximate it using particles at step $l$ by
\begin{equation}\label{eq:Hl}
    \hat H_l =\frac{1}{N} \sum_{n=1}^N \nabla \log f(x^n_l) (\nabla \log f(x^n_l))^T.
\end{equation}
Instead of computing the matrix $\hat H$ and the projection basis at every step, we compute them every $L$ steps to save computational cost. Besides, we approximate the optimal profile function $g^*(x)$ via $f(P_rx)$ since $f$ has negligible change in the complement space \cite{ZahmCuiLawEtAl18}.

The overall projected Wasserstein gradient descent (pWGD) algorithm is summarized in Algorithm \ref{alg:pwgd}, where we can set convergence criteria as small particle move and/or the largest number of iterations reached. Note that for convenience we fix $x_0^{n, \perp}$ (computed line 5) in the complement space of each reconstructed subspace with basis $\Psi_r$, since it is not informed by data. Thanks to the parallel structure of the updating the samples, we also develop a parallel implementation of this algorithm. 


\begin{algorithm}[!htp]
\caption{Projected Wasserstein gradient descent}
\label{alg:pwgd}
\begin{algorithmic}[1]
\REQUIRE initial particles $\{x_0^n\}_{n=1}^N$ drawn from the prior.
\STATE Set $l=0$.
\WHILE{not converged}
\IF {$l \equiv L \text{ (mod) }0$}
\STATE Construct basis $\Psi_r$ by \eqref{eq:genEignen} with $\hat{H}_l$ in \eqref{eq:Hl}. 
\STATE Compute $w_l^n = \Psi_r^T x_l^n$ and $x_0^{n, \perp} = x_l^n - \Psi_r w_l^n$.
\ENDIF
\STATE Update the WGD particle system by \eqref{equ:pwgd}. 
\STATE Set $x_{l+1}^n = \Psi_r w_{l+1}^n + x_0^{n, \perp}$ and $l = l + 1$.
\ENDWHILE
\RETURN $\{x_l^n\}_{n=1}^N$
\end{algorithmic}
\end{algorithm}

\subsection{Batched kernel density estimate}
For some problems, the dimension $r$ after projection can be relatively large for KDE to accurately approximate the density. Hence, we can further consider the hyper-projection operator $P^{(1)},P^{(2)},\dots,P^{(s)}:\mbR^r\to \mbR^r$ which satisfies $P^{(1)}+\dots+P^{(s)}=I_r$. For $j=1,\dots,s$, we denote
$$
\xi_{r,j}(w)=\frac{\sum_{n=1}^N\nabla k^{r,j}(P^{(j)}w,P^{(j)}w_n^l)}{\sum_{n=1}^Nk^{r,j}(P^{(j)}w,P^{(j)}w_n^l)}.
$$
Here $k^{r,j}$ is a Gaussian kernel whose bandwidth is calculated based on $\{P^{(j)}w_l^n\}_{n=1}^N$. In each iteration, we iteratively use KDE to update
$$
w^n_{l,j+1} = w^n_{l,j}+\alpha_l(P^{(j+1)}\nabla_w \log \tilde \pi(w^n_{l,j})-\xi^{r,j+1}(w^n_{l,j}))
$$
with $w^n_{l,0}=w^n_l$. Then, we update $w^n_{l+1}=w^n_{l,s}$. 

\subsection{Projected Wasserstein gradient flow in full space}
For fixed projection operator, pWGD in continuous time corresponds to the projected Wasserstein gradient flow:
$$
\p_t \tilde \rho_t = \nabla\cdot(\tilde \rho_t(\nabla \log \tilde \rho_t-\nabla \log \tilde \pi)),
$$
with $\tilde \rho_t|_{t=0}=\tilde p_0$. Here $\tilde \rho_t$ is the density of $w_t$. In continuous time, the particle dynamics of pWGD in terms of $w$ follows
$$
\frac{dw_t}{dt} = \nabla\log \tilde \pi(w_t)-\nabla \log \tilde \rho_t(w_t).
$$

To characterize the dynamics of pWGD in the full space, we consider the situation that $x_t = \Psi_r w_t+x^\perp_0$, where $x^\perp_0\sim p_0^\perp(\cdot|\Psi_r^T z)$. Here $z\in \mbR^d$ is a fixed vector. Under this assumption, $P_r x_t=\Psi_r w_t$ and $(I-P_r)x_t=x^\perp_0$ are independent. Denote $\rho_t$ as the density function of $x_t$. Then, we have
$$
\rho_t(x)=\tilde \rho_t(\Psi_r^T x) p_0^\perp(x-P_r x|\Psi_r^Tz). 
$$
We obtain the following Wasserstein gradient flow for $\rho_t$ in the full space, whose proof is provided in Appendix A.
\begin{proposition}
Suppose that $x_t = \Psi_r w_t+x^\perp_0$, where $x^\perp_0\sim p_0^\perp(\cdot|\Psi_r^T z)$ for a fixed $z\in \mbR^d$. Let $\rho_t$ be the density function of $x_t$. Then, the evolution of $\rho_t$ satisfies
\begin{equation}\label{pwgf:full}
    \p_t \rho_t = -\nabla\cdot(\rho_t(\nabla \log \hat \pi_r-\nabla \log \rho_t)).
\end{equation}
Here the density $\hat \pi_r(x)$ is defined as
$$
\hat \pi_r(x) = \tilde \pi(\Psi_r^Tx) p_0^\perp(x-P_rx|\Psi_r^Tz). 
$$
\end{proposition}
\begin{proof}
We note that 
$$
\begin{aligned}
    \log \rho_t(x) = &\log \tilde \rho_t(\Psi_r^T x) + \log p_0^\perp(x-P_r x|\Psi_r^Tz)\\
    =&\log \tilde \rho_t(\Psi_r^T x) +\log p_0(x-P_r x+\Psi_r^Tz)-\log p_0^r(\Psi_r^Tz).
\end{aligned}
$$

Therefore, we have
$$
\begin{aligned}
    &\nabla_x \log \rho_t(x) \\
    =&\nabla_x(\log \tilde \rho_t(\Psi_r^T x))+\nabla_x(\log p_0(x-P_r x+\Psi_r^Tz))\\
    =&\Psi_r \nabla_w \log \tilde \rho_t(\Psi_r^T x)+(I-P_r)\nabla_x \log p_0(x-P_r x+\Psi_r^Tz).
\end{aligned}
$$
We also note that
$$
\begin{aligned}
\log \hat \pi_r (x) =& \log \tilde \pi(\Psi_r^T x)+ \log p_0^\perp(x-P_rx|\Psi_r^Tz)\\
=&\log \tilde \pi(\Psi_r^T x)+\log p_0(x-P_r x+\Psi_r^Tz)-\log p_0^r(\Psi_r^Tz).
\end{aligned}
$$
Similarly, this yields
$$
\begin{aligned}
&\nabla_x\log \hat \pi_r(x) \\
= &\Psi_r\nabla_w \log \tilde \pi(\Psi_r^T x)+(I-P_r)\nabla \log p_0(x-P_r x+\Psi_r^Tz).
\end{aligned}
$$
Based on these observations, we have the identity:
$$
\begin{aligned}
    &\nabla_x\log \hat \pi_r(x)-\nabla_x \log \rho_t(x)\\
=&\Psi_r\nabla_w \log \tilde \pi(\Psi_r^T x)-\Psi_r \nabla_w \log \tilde \rho_t(\Psi_r x).
\end{aligned}
$$
In summary, we have
$$
\begin{aligned}
\frac{dx_t}{dt} = &\Psi_r\frac{d w_t}{dt}\\
=&\Psi_r\pp{\nabla_w\log \tilde \pi(\Psi_r^T x_t)-\nabla_w \log \tilde \rho_t(\Psi_r^T x_t)}\\
=&\nabla_x\log \pi_r(x_t)-\nabla_x \log \rho_t(x_t).
\end{aligned}
$$
Hence, as the density function of $x$, $\rho_t$ satisfies 
$$
\p_t \rho_t = -\nabla\cdot(\rho_t(\nabla \log \hat \pi_r-\nabla \log \rho_t)).
$$

\end{proof}


\subsection{Convergence analysis in continuous time}
Suppose that $-\log \pi$ is  $\mu$ strongly convex, where $\mu>0$. From the classic analysis \citep{otoan} [Theorem 24.7], the Wasserstein gradient flow \eqref{equ:wgf} has the exponential convergence rate:
$$
\mathrm{D}_{\mathrm{KL}}( \rho_t| \pi) = O(\exp(- \mu t)).
$$
For fixed projection operator, pWGD in continuous time corresponds to the projected Wasserstein gradient flow:
$$
\p_t \tilde \rho_t = \nabla\cdot(\tilde \rho_t(\nabla \log \tilde \rho_t-\nabla \log \tilde \pi)),
$$
with $\tilde \rho_t|_{t=0}=\tilde p_0$. Assume that $-\log \tilde \pi$ is $\tilde \mu$ strong convex, where $\tilde \mu>0$. Similarly, the projected Wasserstein gradient flow has the following exponential convergence rate:
$$
\mathrm{D}_{\mathrm{KL}}(\tilde \rho_t|\tilde \pi) = O(\exp(-\tilde \mu t)).
$$
Under suitable assumptions we can show that the convergence of projected Wasserstein gradient flow is at least as fast as the convergence of Wasserstein gradient flow in the full space. The proof is provided in Appendix B. 
\begin{proposition}\label{prop:conv}
Suppose that the prior $p_0$ follows an isotropic Gaussian distribution. It satisfies that $-\nabla^2 \log p_0(x)= \xi I$. Assume that $\sigma_\text{min}\pp{-\frac{\nabla^2 f(x)}{f(x)}}\geq \mu-\xi$. Then, we have $\tilde \mu\geq \mu$.
\end{proposition}
\begin{proof}
First, we note that
$$
\log \pi = \log f+\log p_0-\log Z. 
$$
Because $-\log \pi$ is  $\mu$ strongly convex, this suggests that for all $x\in \mbR^d$,
$$
\sigma_\text{min}(-\nabla^2 \log f(x) -\nabla^2 \log p_0(x))\geq \mu,
$$
where $\sigma_\text{min}(A)$ denotes the smallest eigenvalue of a symmetric matrix $A$. We note that
$$
-\nabla^2 \log f(x) = -\frac{\nabla^2 f(x)}{f(x)}+\frac{\nabla f(x)(\nabla f(x))^T}{f(x)^2}. 
$$
Because $\frac{\nabla f(x)(\nabla f(x))^T}{f(x)^2}$ is a rank-1 matrix, we can assume that $\sigma_\text{min}\pp{-\frac{\nabla^2 f(x)}{f(x)}}\geq \mu-\xi$. 
On the other hand, we note that 
$$
\log \tilde \pi(w) = \log g(\Psi_r w)+\log p_0^r(\Psi_rw)-\log \tilde Z.
$$
Thus, $\tilde \mu$ shall satisfy that for all $w\in \mbR^r$,
$$
\sigma_\text{min}(-\Psi_r^T(\nabla^2 \log g(\Psi_r w) +\nabla^2 \log p_0^r(\Psi_r w))\Psi_r )\geq \tilde \mu .
$$
We note that
$$
\begin{aligned}
&-\nabla^2 \log g(\Psi_r w)\\
= &-\frac{\nabla^2 g(\Psi_r w)}{g(\Psi_r w)} +\frac{\nabla  g(\Psi_r w)(\nabla  g(\Psi_r w))^T}{g(\Psi_r w)^2}.
\end{aligned}
$$
For $v\in \mbR^r$, we have
$$
\begin{aligned}
&-v^T\Psi_r^T \nabla^2 \log g(\Psi_r w) \Psi_r v\\
=&-\frac{v^T\Psi_r^T  \nabla^2 g(\Psi_r w) \Psi_r v}{g(\Psi_r w)}+\frac{|v^T\Psi_r^T\nabla  g(\Psi_r w)|^2 }{g(\Psi_r w)^2}\\
\geq &-\frac{v^T\Psi_r^T  \nabla^2 g(\Psi_r w) \Psi_r v}{g(\Psi_r w)}
\end{aligned}
$$
Recall that
$$
\nabla^2 g(\Psi_r w) = \int_{z\in X^\perp} \nabla^2 f(\Psi_r w+z)p_0^\perp(z)dz
$$
This indicate that
$$
\begin{aligned}
&-v^T\Psi_r^T  \nabla^2 g(\Psi_r w) \Psi_r v \\
=&-\int_{z\in X^\perp} v^T\Psi_r^T\nabla^2 f(\Psi_r w+z)\Psi_r vp_0^\perp(z)dz\\
\geq &\int_{z\in X^\perp} (\mu-\xi)f(\Psi_r w+z)\|\Psi_r v\|_2^2p_0^\perp(z)dz\\
=& (\mu-\xi)\|v\|_2^2 g(\Psi_r).
\end{aligned}
$$
This indicates that $\sigma_\text{min}(\Psi_r^T\nabla^2 \log g(\Psi_r w)\Psi_r)\geq \mu-\xi$. Besides, because $-\nabla^2 \log p_0(x)= \xi I$, we also have $-\nabla \log p_0^r (\Psi_r w)=\xi I$. Hence, we have
$$
\sigma_\text{min}(-\Psi_r^T(\nabla^2 \log g(\Psi_r w) +\nabla^2 \log p_0^r(\Psi_r w))\Psi_r )\geq \mu.
$$
Thus, we have $\tilde \mu \geq \mu$.
\end{proof}

\subsection{Approximation error of optimal profile function}
Suppose that information matrix $H$ defined in \eqref{equ:h} has a sharp eigenvalue decay. This can make $r$ significantly smaller than $d$ and this also yields that for $z\in X^\perp$ with $\|z\|_{\Gamma}=1$,
\begin{equation}\label{equ:exp}
    \mbE_{x\sim \pi} |\nabla \log f(x)^Tz|^2=z^THz\leq \lambda_{r+1}.
\end{equation}
Hence, we can assume that the following statement holds: 
\begin{assumption}\label{asp:proj}
Given the projection operator $P_r$, for all $w\in \mbR^r$ and $z_1,z_2\in X^\perp$, there exists $\epsilon_1>0$ such that
$$
|\nabla \log f(\Psi_r w+z_2)^Tz_1|\leq \epsilon_1\|z_1\|_{\Gamma},
$$
\end{assumption} 
Based on this assumption, we have the following estimation for the difference between the optimal profile function $g(\Psi_r w)$ and $f(\Psi_r w)$, whose proof is provided in Appendix C. 
\begin{proposition}
Under Assumption \eqref{asp:proj}, denote $0<\delta_2<1<\delta_1$ as
$$
\delta_1 = \sup_{w\in \mbR^r}\int_{z\in X^\perp} \exp\pp{\frac{\epsilon_1}{2}\|z\|_{\Gamma}} p_0^\perp(z|\Psi_r w)dz,
$$
$$
\delta_2 = \inf_{w\in \mbR^r}\int_{z\in X^\perp} \exp\pp{-\frac{\epsilon_1}{2}\|z\|_{\Gamma}} p_0^\perp(z|\Psi_r w)dz.
$$
Then, for all $w\in \mbR^r$, 
$$
\delta_2 \leq\frac{g(\Psi_r w)}{f(\Psi_r w)}\leq \delta_1.
$$
\end{proposition}
\begin{proof}
Under Assumption 1, for all $w\in \mbR^r$ and $z\in X^\perp$, we have:
$$
\begin{aligned}
&|\log f(\Psi_r w+z)-\log f(\Psi_r w)| \\
=& \left|\int_0^{1} sz^T\nabla \log f(\Psi_r w+sz)ds\right|\\
\leq & \int_0^{1} s |z^T\nabla \log f(\Psi_r w+sz)|ds\\
\leq & \epsilon_1\int_0^1 s\|z\|_{\Gamma}=\frac{\epsilon_1}{2}\|z\|_{\Gamma}.
\end{aligned}
$$
This indicates that
$$
\begin{aligned}
    &g(\Psi_r w)/f(\Psi_r w)\\
=&\int_{z\in X^\perp} \frac{f(\Psi_r w+z)}{f(\Psi_r w)}p_0^\perp(z|\Psi_r w)dz\\
=&\int_{z\in X^\perp} \exp(\log f(\Psi_r w+z)-\log f(\Psi_r w) )p_0^\perp(z|\Psi_r w)dz\\
\leq&\int_{z\in X^\perp} \exp\pp{\frac{\epsilon_1}{2}\|z\|_{\Gamma}} p_0^\perp(z|\Psi_r w)dz\\
\leq &\delta_1,
\end{aligned}
$$
where $\delta_1>0$ is a constant. Besides, we also have
$$
\begin{aligned}
    &g(\Psi_r w)/f(\Psi_r w)\\
\geq&\int_{z\in X^\perp} \exp\pp{-\frac{\epsilon_1}{2}\|z\|_{\Gamma}} p_0^\perp(z|\Psi_r w)dz\\
\geq &\delta_2,
\end{aligned}
$$
where $\delta_2\in (0,1)$ is a constant.
\end{proof}
From \eqref{equ:exp}, $\epsilon_1$ in Assumption 1 can be close to zero. We also note that for small $\epsilon_1$, the constants $\delta_1,\delta_2$ are close to $1$. 

\section{Numerical experiments}\label{sec:num}
In this section, we present a variety of numerical experiments to demonstrate the accuracy, convergence, and scalability of pWGD compared to WGD, SVGD, and pSVGD. The code for all the results is available at \href{https://github.com/cpempire/pWGD}{https://github.com/cpempire/pWGD}.

\subsection{Toy examples}
We first present two toy examples. The first example is a bi-modal posterior distribution with a Gaussian prior. WGD-MED and WGD-BM denote WGD with kernel bandwidth calculated by the MED method \citep{SVGD} and the BM method \citep{aigf} respectively. We compare WGD-MED, WGD-BM with SVGD. The results are presented in Figure \ref{fig:toy2d}. We note that WGD converges much faster than SVGD. Besides, WGD-BM captures the variance of the posterior distribution better than WGD-MED. 

\begin{figure}[htb]
\centering
\begin{minipage}[t]{0.9\textwidth}
\centering
\includegraphics[width=\linewidth]{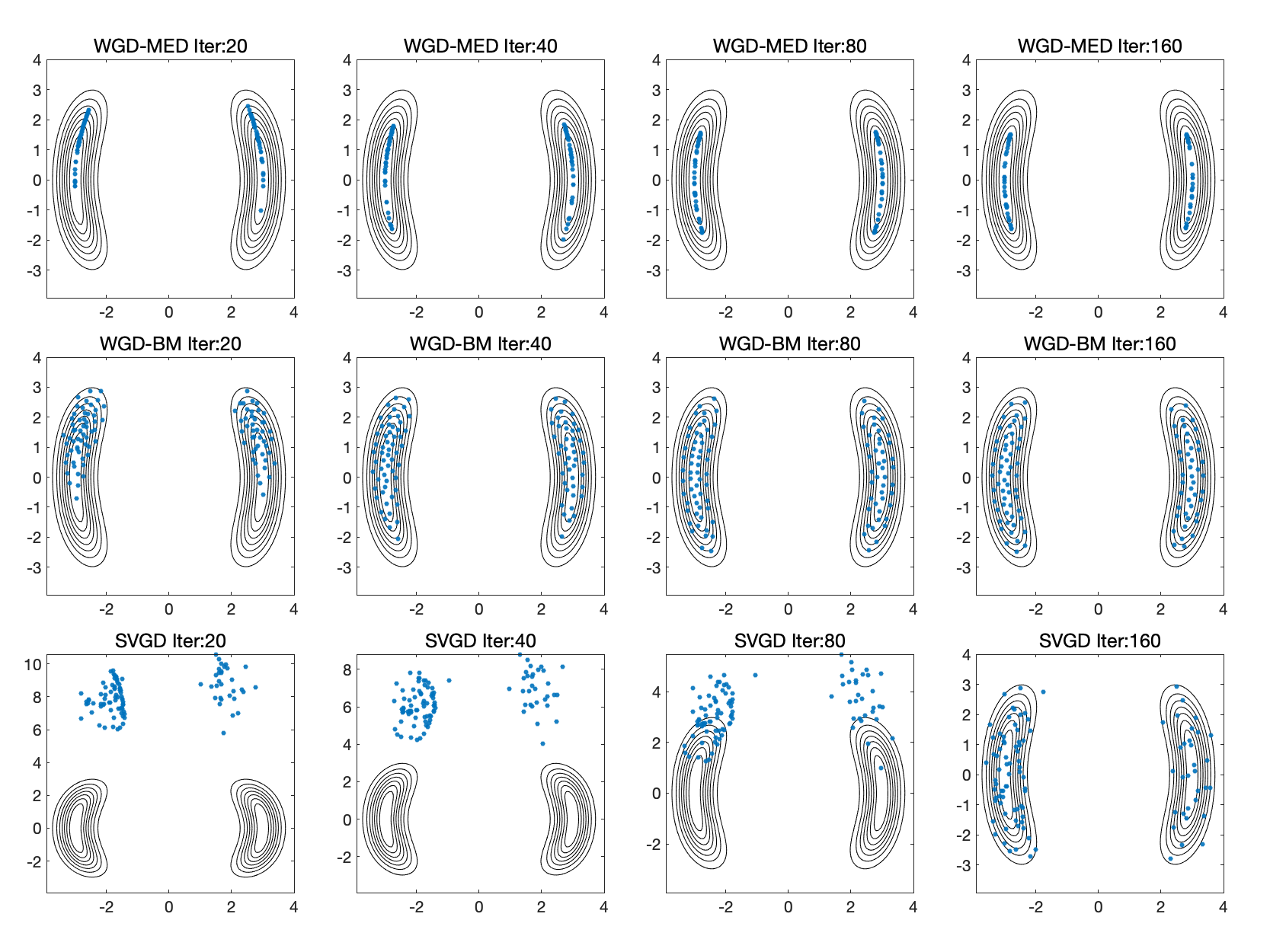}
\end{minipage}
\caption{Comparison of WGD-MED, WGD-BM and SVGD on a toy example for their convergence and accuracy.} \label{fig:toy2d}
\end{figure}

Another example is a double-banana-shaped posterior distribution with a Gaussian prior. The results are presented in Figure \ref{fig:toy_db}. We note that WGD is still faster than SVGD. WGD-MED and WGD-BM have similar performance.

\begin{figure}[htb]
\centering
\begin{minipage}[t]{0.9\textwidth}
\centering
\includegraphics[width=\linewidth]{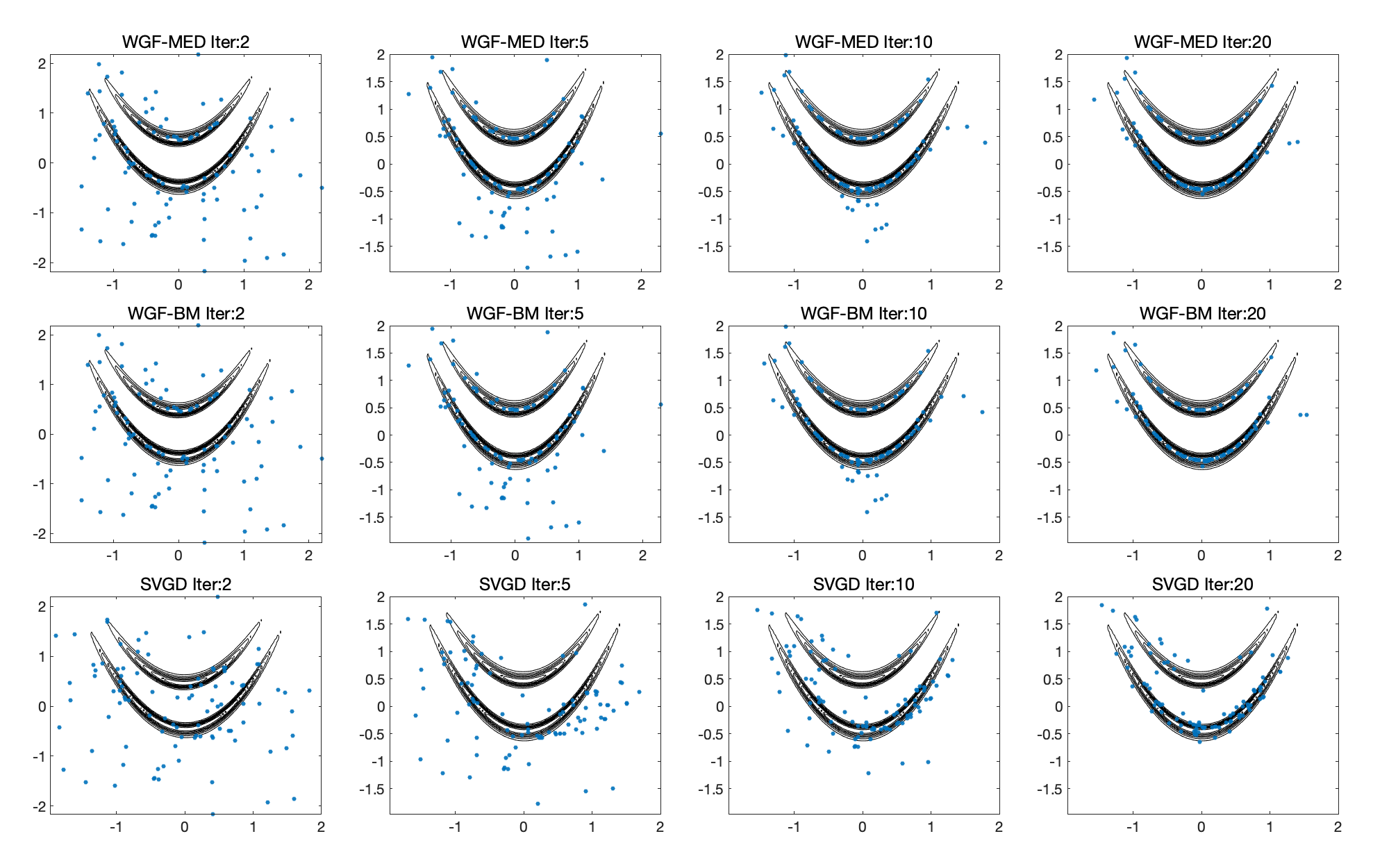}
\end{minipage}
\caption{Comparison of WGD-MED, WGD-BM and SVGD on a double-banana example for their convergence and accuracy.} \label{fig:toy_db}
\end{figure}

\subsection{Linear Bayesian inference}
We consider the following equation as a model for contaminant diffusion in environmental engineering
$$
- \kappa \Delta u + \nu u = x \quad \text{ in } D,
$$
where $D$ is a physical domain, $x$ is an infinite-dimensional contaminant source field parameter to be inferred, $u$ is the contaminant concentration which we can observe at some locations, $\kappa$ and $\nu$ are diffusion and reaction coefficients. Thanks to the linearity of the parameter $x$ to the observable $u$, we have a linear Bayesian inference problem. Under the assumption of Gaussian prior for $x \sim \mathcal{N}(x_0, C)$ and Gaussian observation noise, we have a Gaussian posterior whose mean and covariance can be explicitly given. For simplicity, we set $\kappa, \nu = 1$, $D = (0, 1)$, $u(0)=u(1)=0$, and consider 15 pointwise observations of $u$ with $1\%$ noise, equidistantly distributed in $D$. We set $x_0 = 0$ and use a covariance given by differential operator $C=(-\delta\Delta + \gamma I)^{-\alpha}$ with $\delta, \gamma, \alpha > 0$ representing the correlation length and variance, which is commonly used in geoscience \citep{LindgrenRueLindstroem11}. We set $\delta = 0.1, \gamma = 1, \alpha = 1$. We solve this forward model by a finite element method with piece-wise elements on a uniform mesh of size $2^k$ where $k=4,6,8$, leading to dimension $17, 65, 257$ for the discrete $x$.   

\begin{figure}[!htbp]
\centering
\begin{minipage}[t]{0.45\textwidth}
\centering
\includegraphics[width=\linewidth]{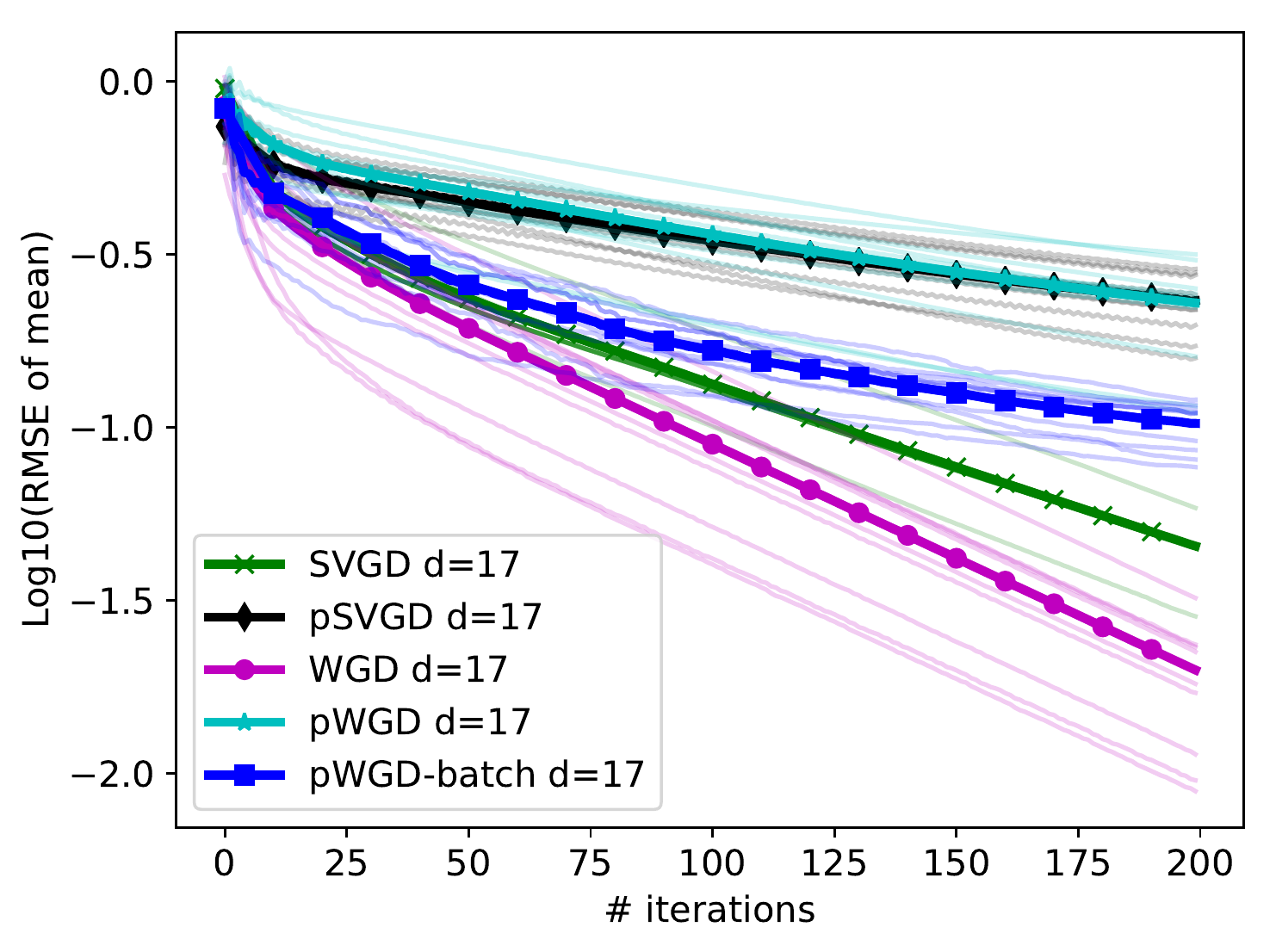}
\end{minipage}
\begin{minipage}[t]{0.45\textwidth}
\centering
\includegraphics[width=\linewidth]{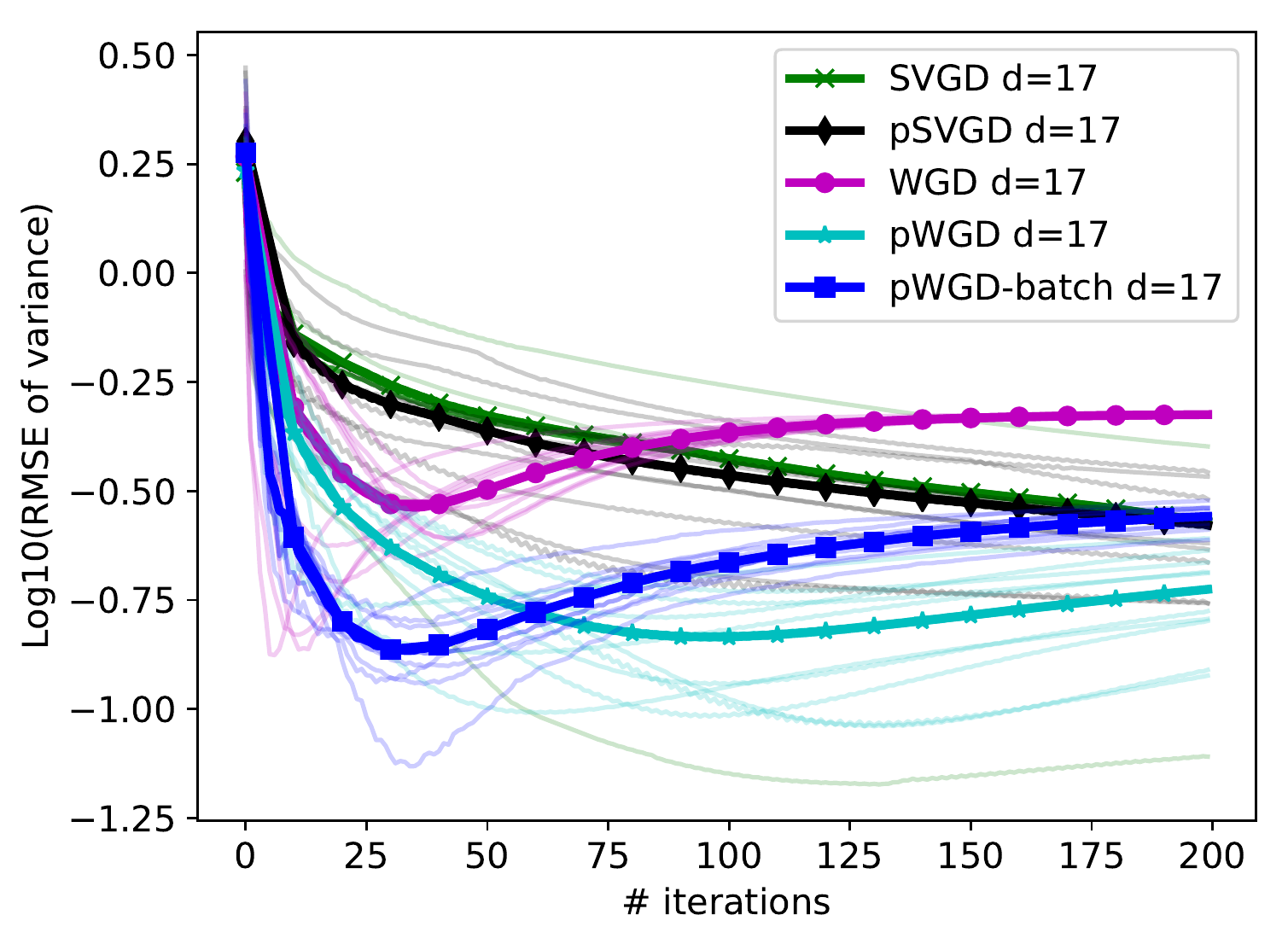}
\end{minipage}
\begin{minipage}[t]{0.45\textwidth}
\centering
\includegraphics[width=\linewidth]{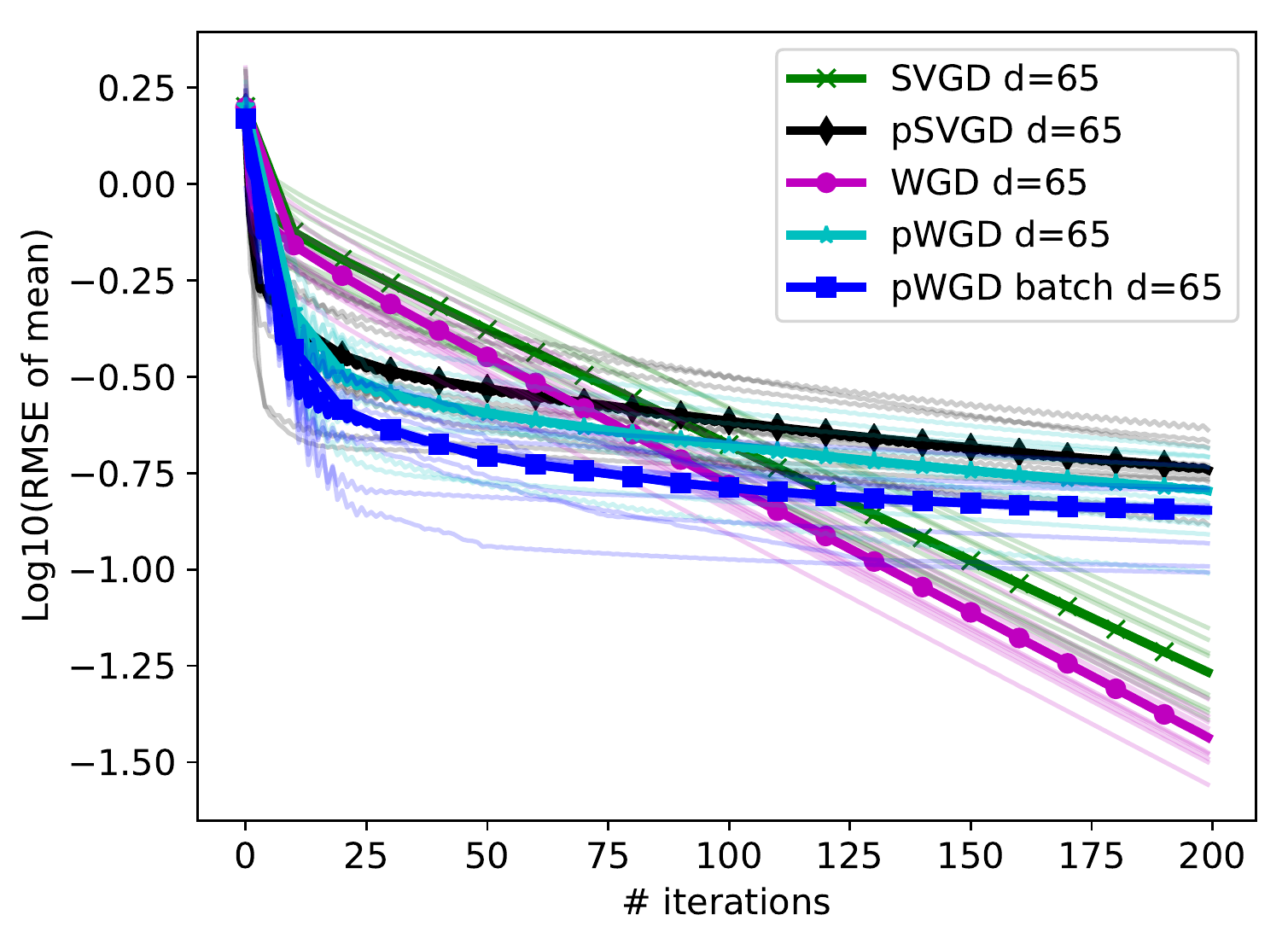}
\end{minipage}
\begin{minipage}[t]{0.45\textwidth}
\centering
\includegraphics[width=\linewidth]{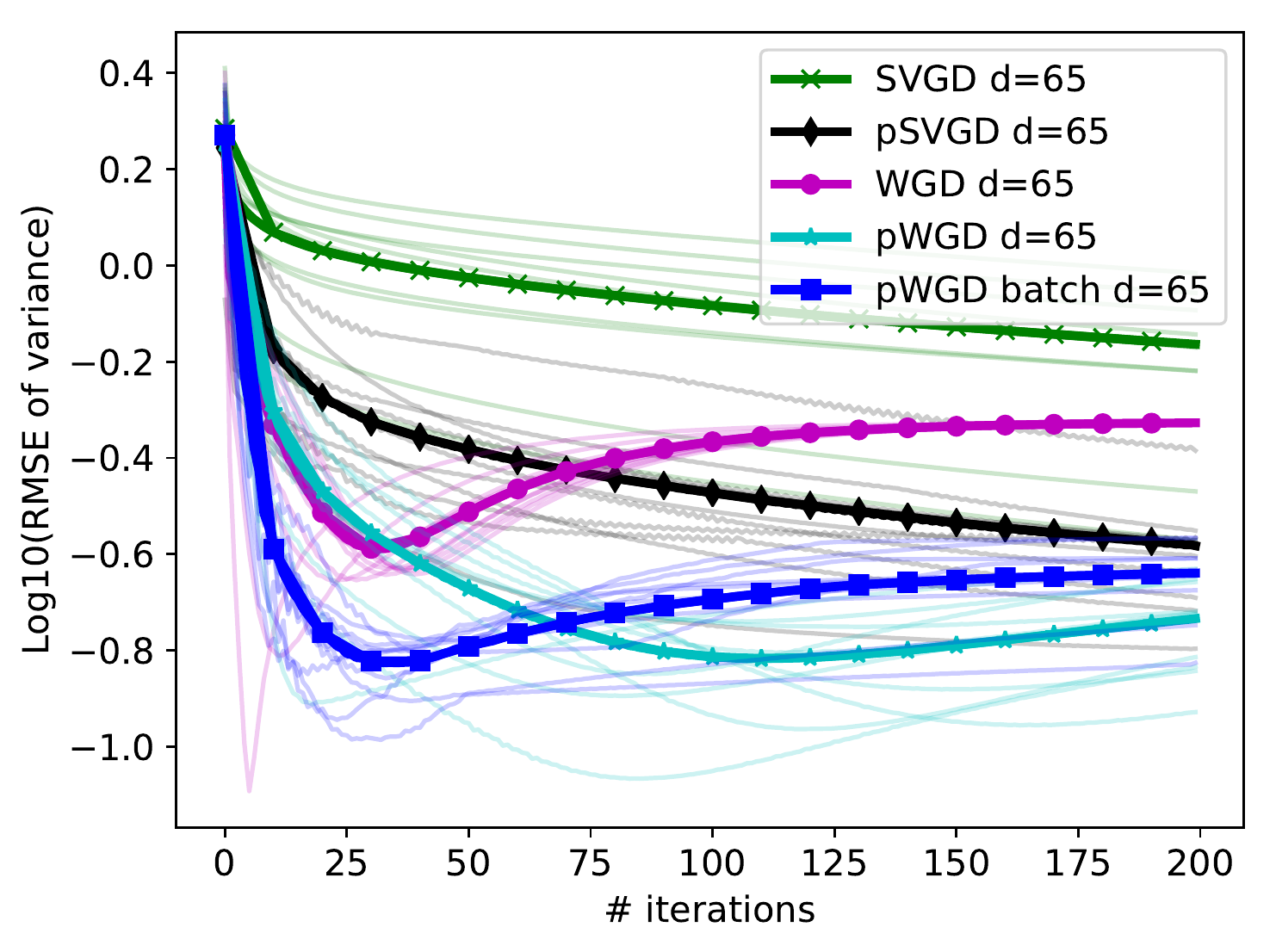}
\end{minipage}
\begin{minipage}[t]{0.45\textwidth}
\centering
\includegraphics[width=\linewidth]{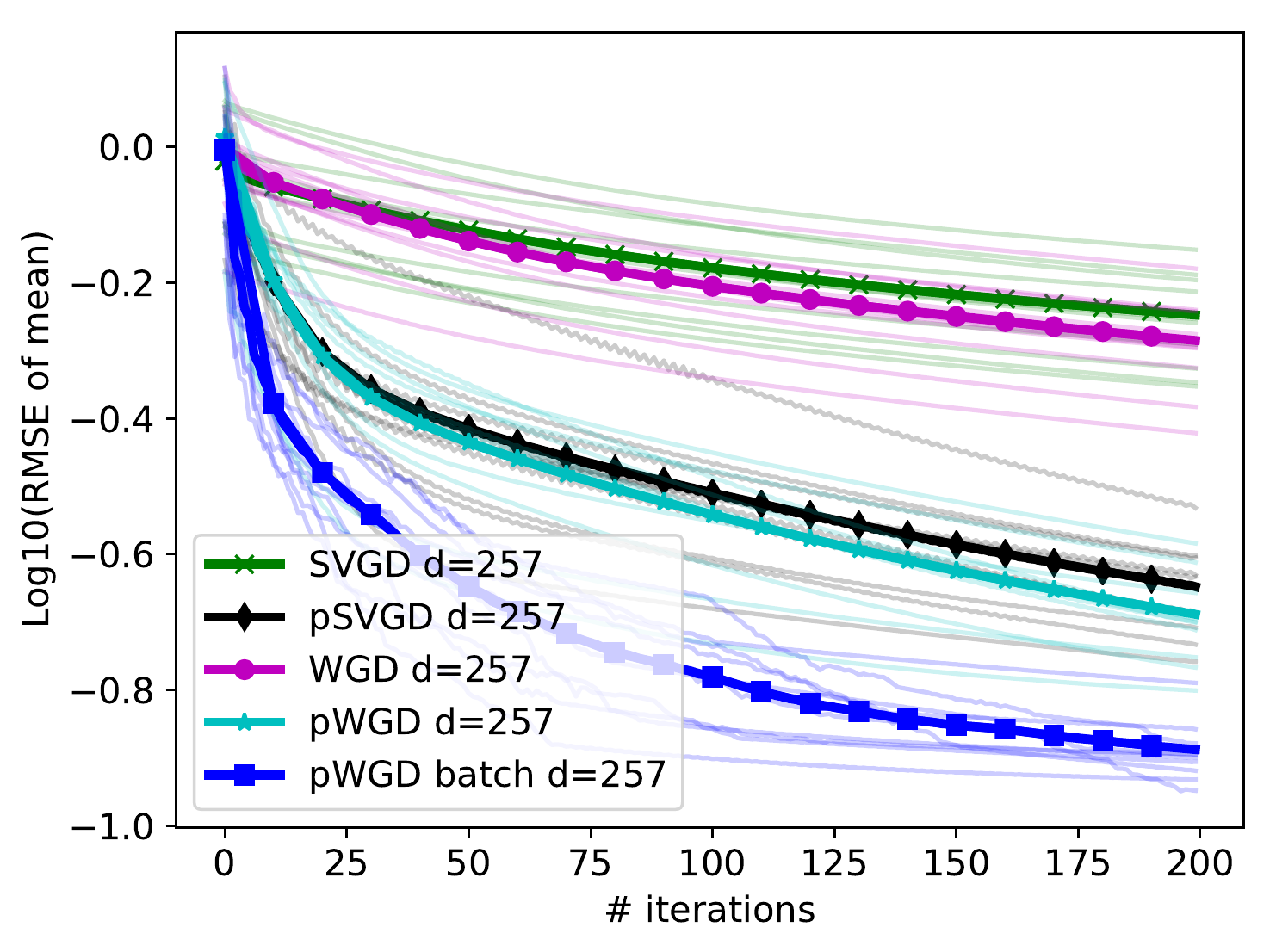}
\end{minipage}
\begin{minipage}[t]{0.45\textwidth}
\centering
\includegraphics[width=\linewidth]{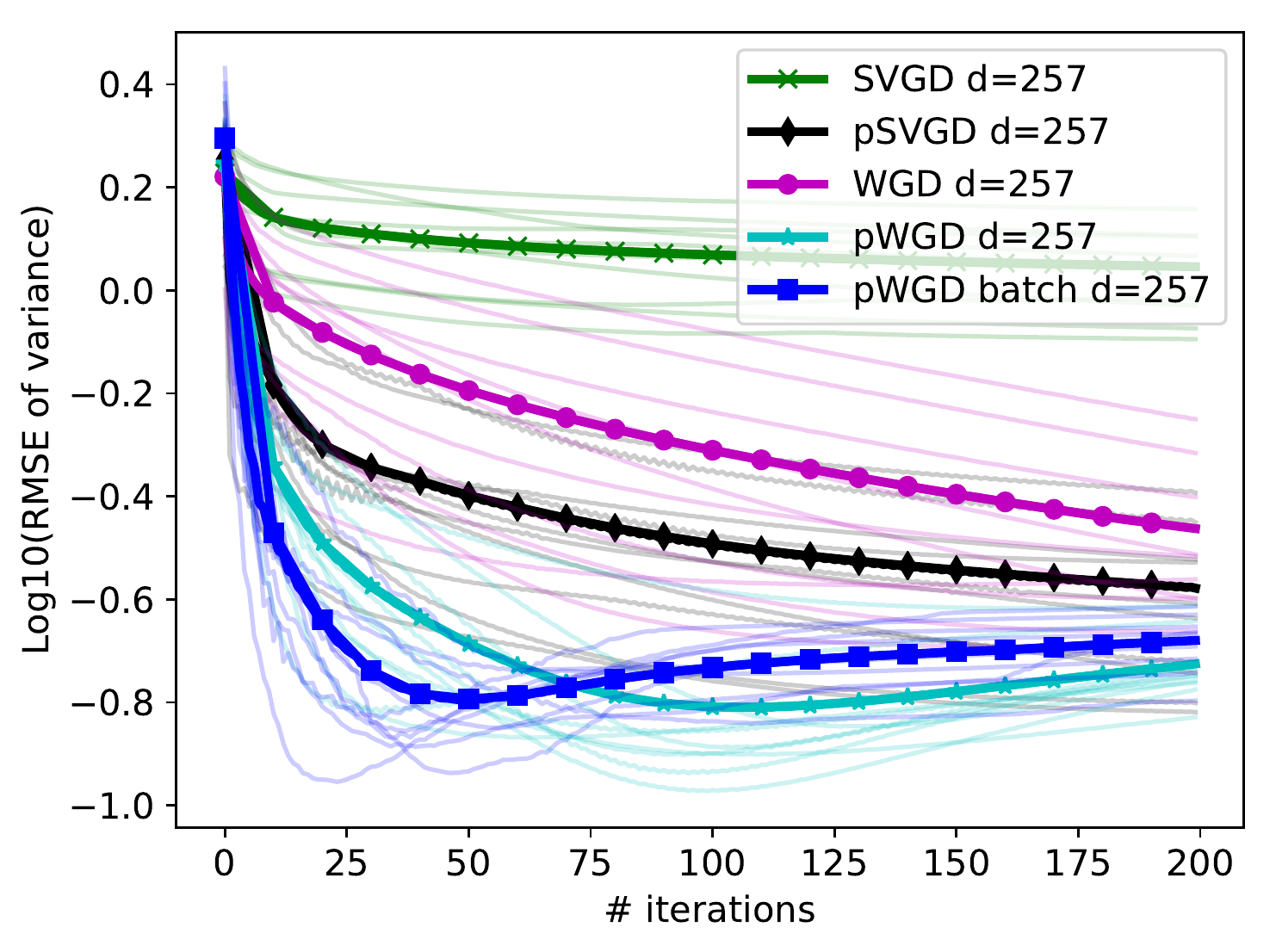}
\end{minipage}
\caption{\small Comparison on accuracy and convergence by RMSE (showing 10 trials and mean) for sample mean (left) and variance (right) among SVGD, pSVGD, WGD, pWGD, pWGD-batch.} \label{fig:linear}
\end{figure}

We compare pWGD, pWGD with batched KDE (pWGD-batch), WGD, SVGD, and pSVGD. For all compared methods, we use a (small) sample size $N=16$, projection tolerance $10^{-4}$ (leading to 8 dimensions of subspace). A smaller batch size ($5 < 8$) is used for pWGD-batch. We evaluate the accuracy of the sampling methods by the $L_2$-norm of the mean and point-wise variance of the parameter $x$ w.r.t. its poster distribution in Figure \ref{fig:linear}, which display the convergence of the root mean square error (RMSE) of the sample mean (left) and variance (right) for dimension $17, 65, 257$. We can observe from the right figures that with increasing dimensions, WGD and SVGD can not capture the variance as the samples collapse to the mean, while pWGD and pSVGD can preserve the accuracy of the variance. Both WGD and pWGD converge faster and achieve higher accuracy than SVGD and pSVGD. From the left figures we can see that pWGD-batch produces more accurate sample mean compared to pWGD with comparable accuracy for variance.

\subsection{Nonlinear Bayesian inference}
In this experiment, we consider a nonlinear Bayesian inference problem constrained by the following PDEs for subsurface (Darcy) flow 
\begin{equation}\label{eq:nonlinear-pde}
\begin{split}
    \mathbf{v} + e^x \nabla u & = 0 \quad \text{ in } D \\
    \nabla \cdot \mathbf{v} & = h \quad \text{ in } D \\
\end{split}
\end{equation}
where $u$ is pressure, $\mathbf{v}$ is velocity, $h$ is force, $e^x$ is the uncertain permeability field equipped with a Gaussian prior $\mathrm{x} \sim \mathcal{N}(x_0, C)$ with $C = (-\delta \Delta + \gamma I)^{-\alpha}$ where we set $\delta = 0.1, \gamma = 1, \alpha = 2$ and $x_0 = 0$. We set $D = (0, 1)^2$ and use a finite element method for the discretization of the problem. The data is generated as pointwise observation at 49 points equidistantly distributed in $(0, 1)^2$ corrupted with additive $5\%$ Gaussian noise. 

\begin{figure}[!htbp]
\centering
\begin{minipage}[t]{0.45\textwidth}
\centering
\includegraphics[width=\linewidth]{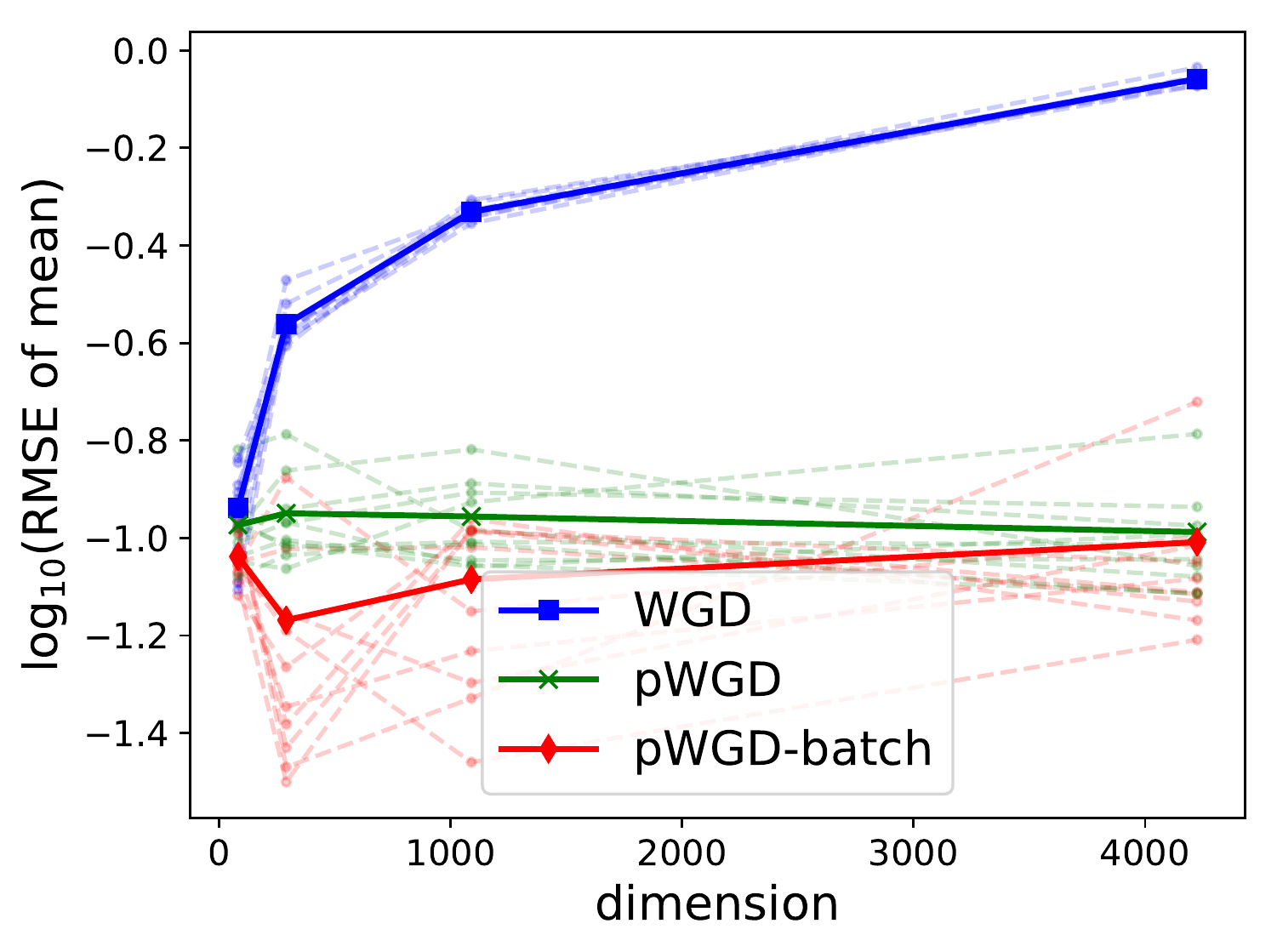}
\end{minipage}
\hspace{0.1cm}
\begin{minipage}[t]{0.45\textwidth}
\centering
\includegraphics[width=\linewidth]{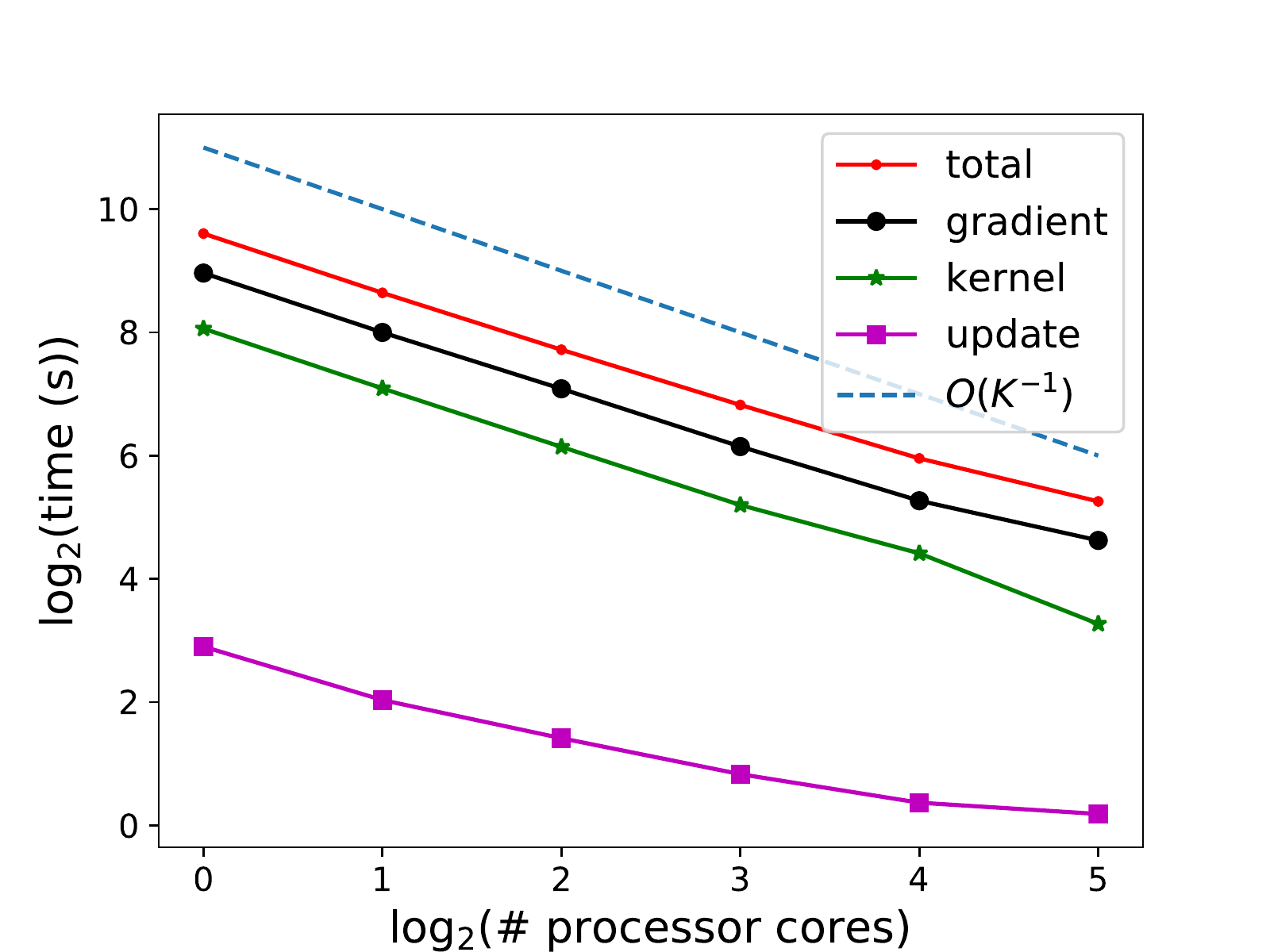}
\end{minipage}
\begin{minipage}[t]{0.45\textwidth}
\centering
\includegraphics[width=\linewidth]{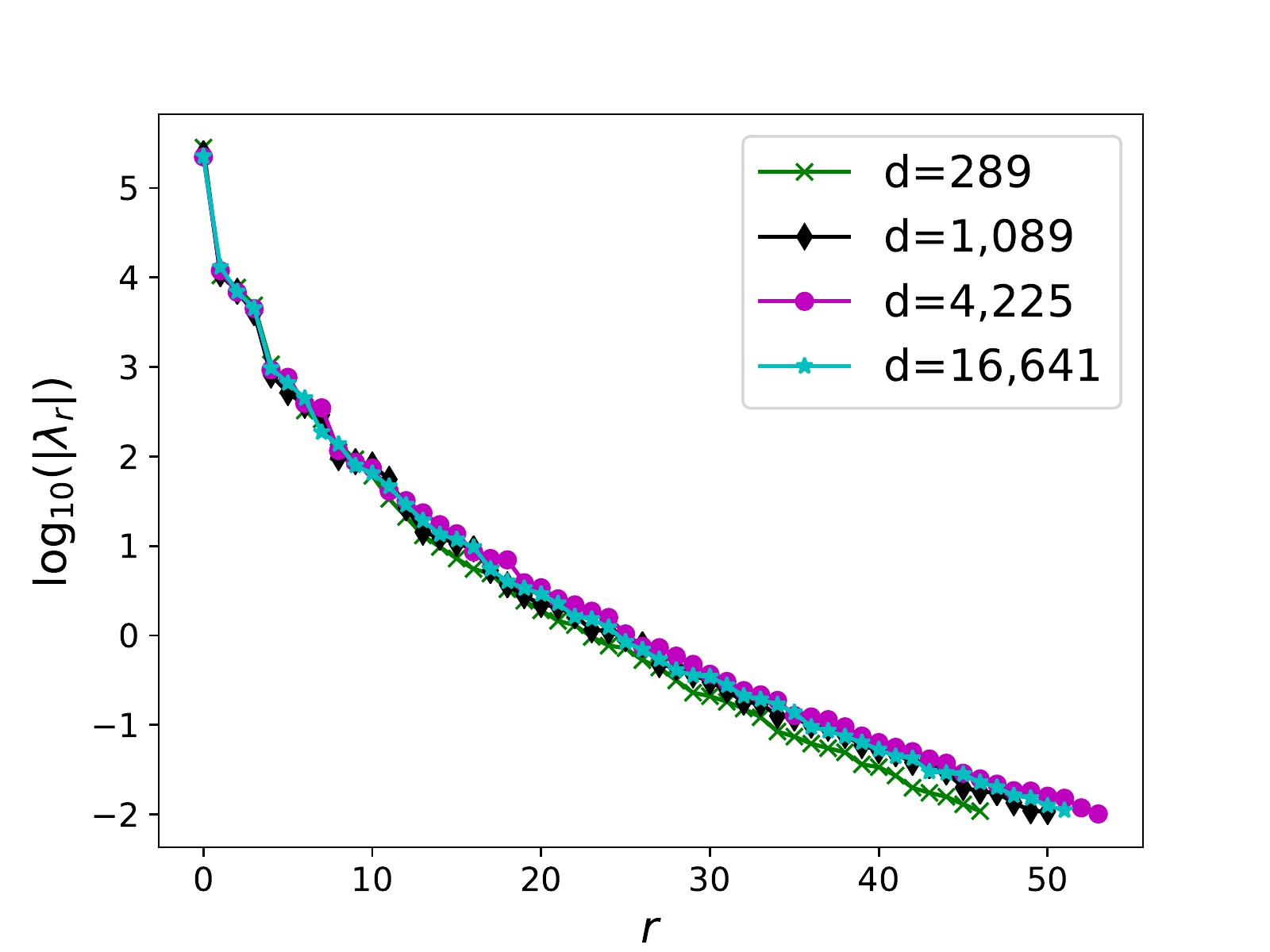}
\end{minipage}
\begin{minipage}[t]{0.45\textwidth}
\centering
\includegraphics[width=\linewidth]{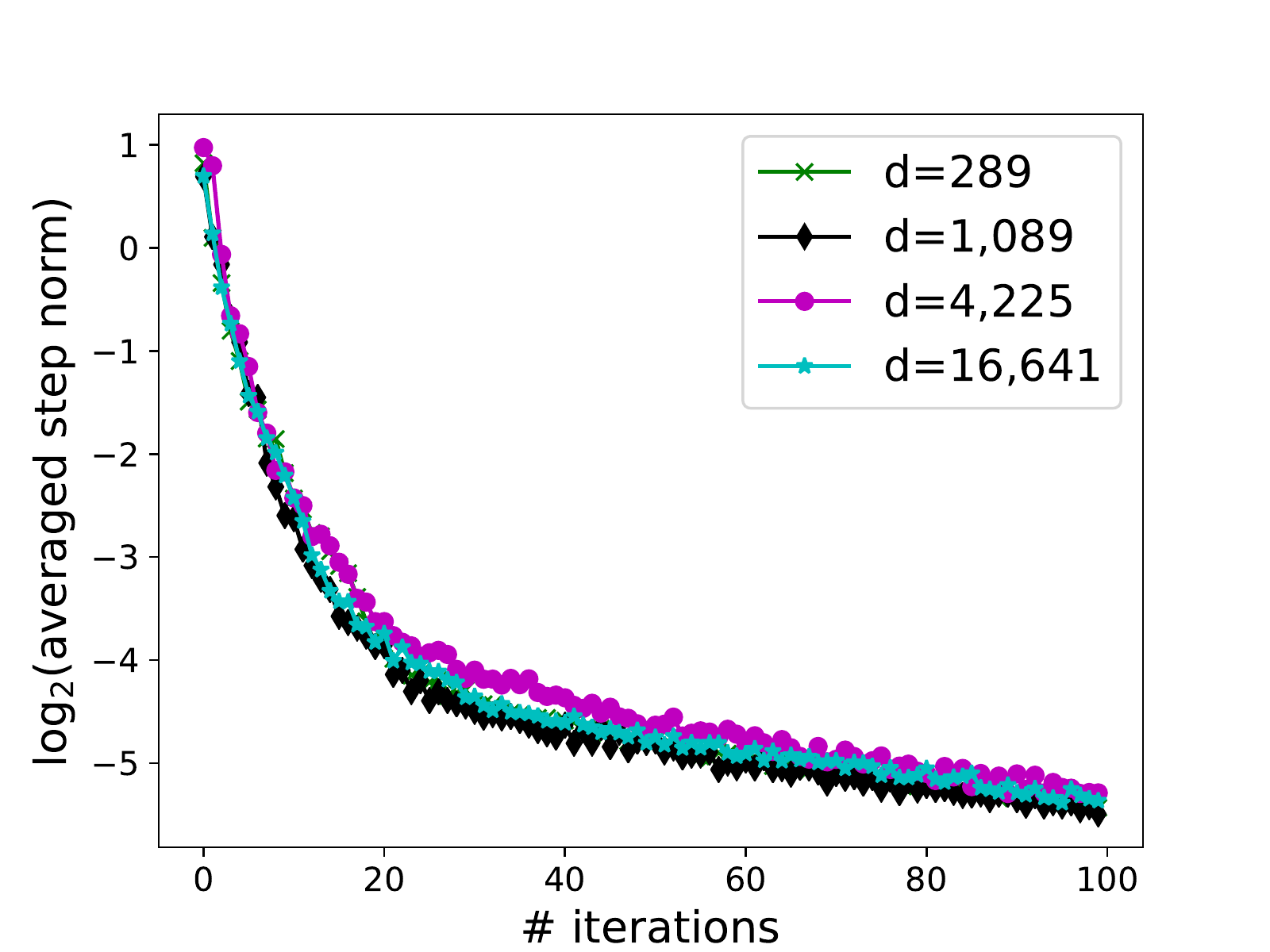}
\end{minipage}
\begin{minipage}[t]{0.45\textwidth}
\centering
\includegraphics[width=\linewidth]{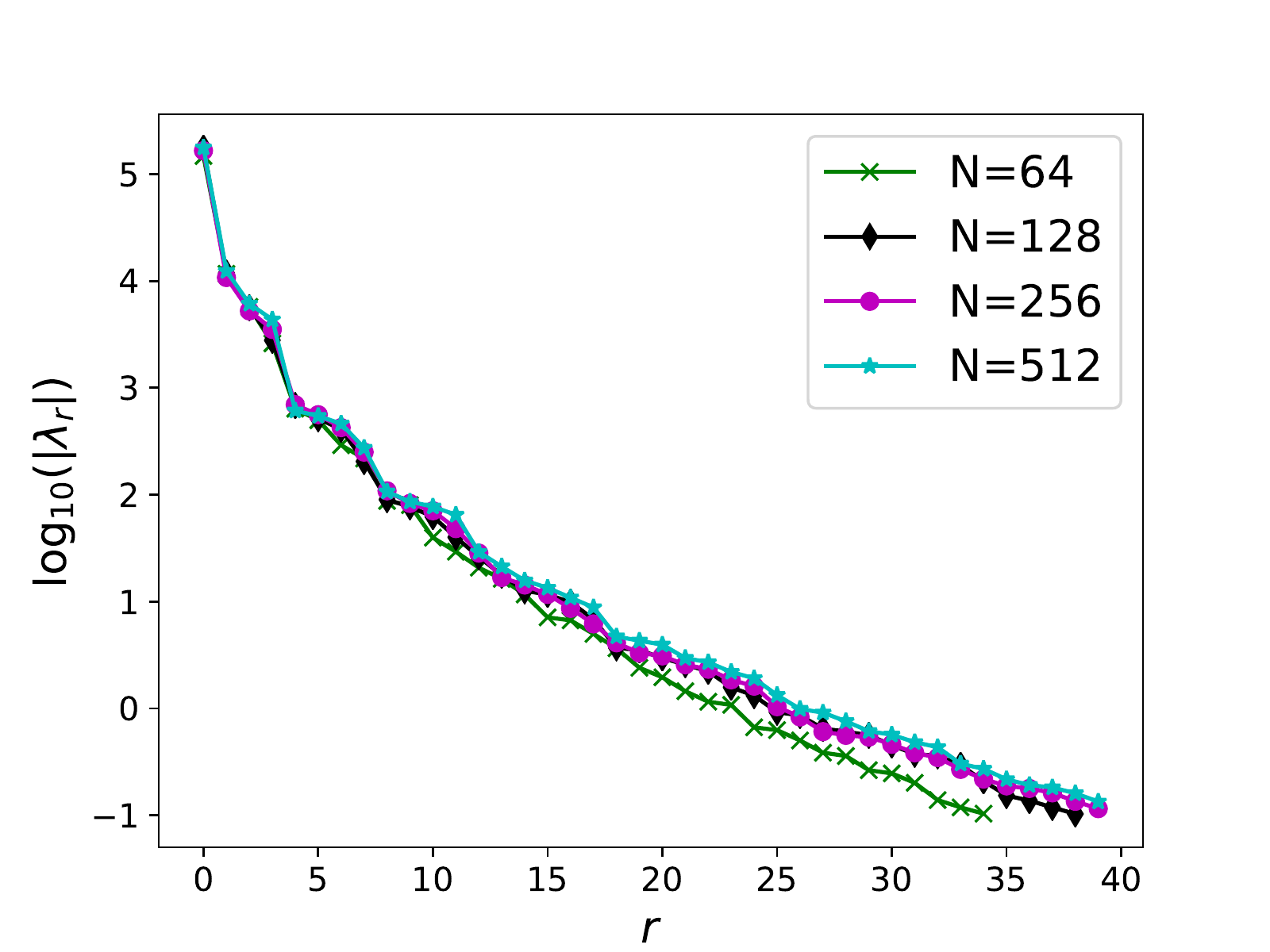}
\end{minipage}
\begin{minipage}[t]{0.45\textwidth}
\centering
\includegraphics[width=\linewidth]{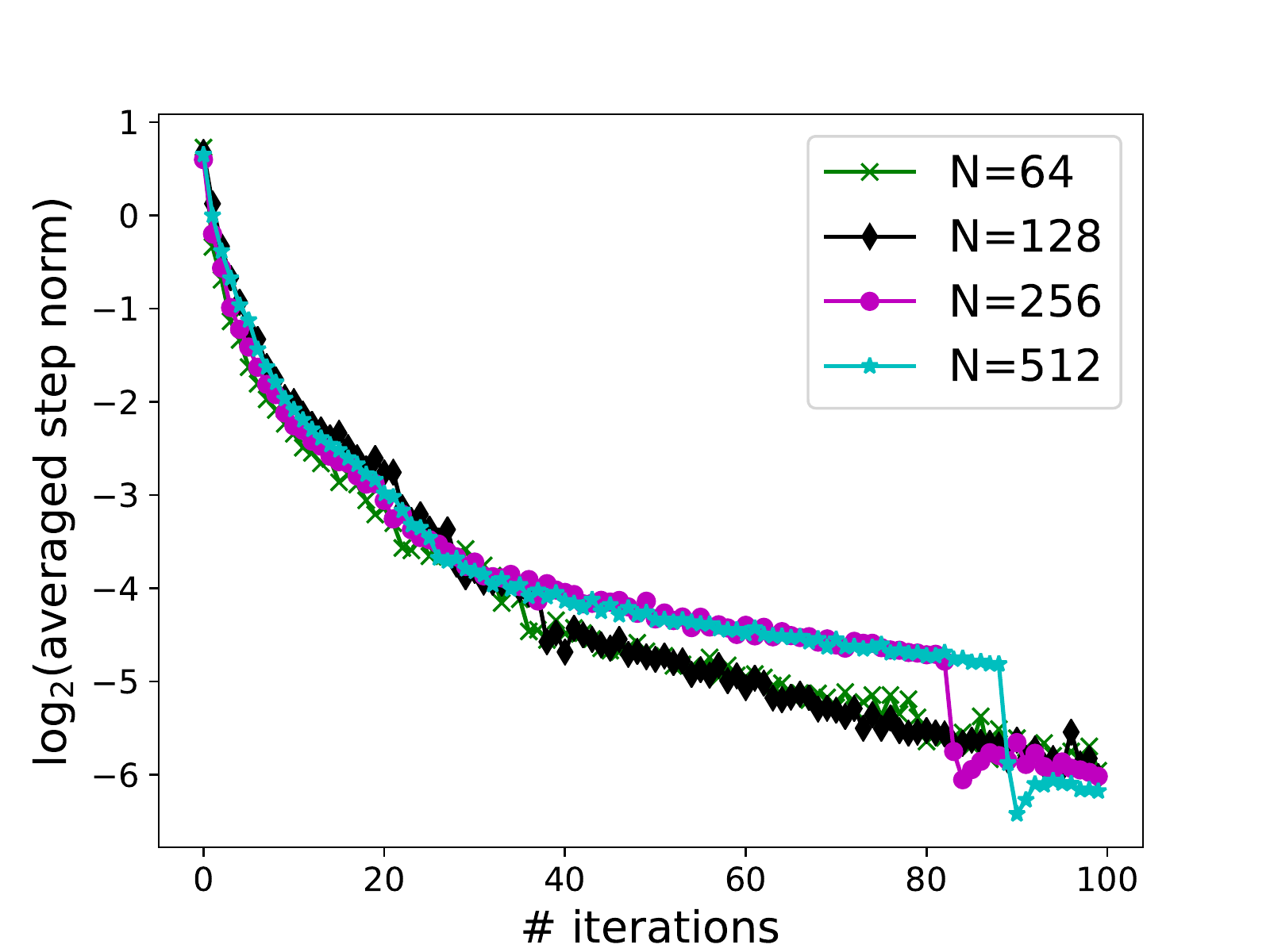}
\end{minipage}
\caption{Scalability of pWGD with respect to the parameter dimension, sample size, and CPU cores in terms of accuracy, decay of eigenvalues and step norms, and computational time.} \label{fig:nonlinear}
\end{figure}

We run the WGD and pWGD algorithms with 200 iterations for different parameter dimensions $d = 9^2, 17^2, 33^2, 65^2$, sample size $N = 64, 128, 256, 512$, in different CPU processor cores $K = 1, 2, 4, 8, 16, 32$ for their parallel implementation. The results are shown in Figure \ref{fig:nonlinear}. From the top-left we can see that both pWGD and pWGD-batch (with batch size 5) preserve the accuracy of the sample mean with 256 samples (compared to a reference value computed by a DILI-MCMC algorithm \citep{CuiLawMarzouk16} with 10000 samples), with the later gives slightly more accurate result, while WGD leads to increasing errors with respect to the parameter dimension. This can be explained by the similar fast decay of eigenvalues in the middle-left figure where the projection dimension does not change much. Moreover, pWGD have similar convergence in averaged sample step norm (norm of sample updates from one step to the next) for different parameter dimensions, as can be seen from the middle-right figure. Similar behavior of the decay of the eigenvalues and step norms can be observed with respect to increasing sample size as shown in the bottom two figures. Finally, as we increase the number of CPU cores, as seen from the top-right figure, the cost for different computational parts is reduced linearly. These results demonstrate the scalability of the parallel pWGD algorithm with respect to the parameter dimension, sample size, and processor cores, which indicate its feasibility to use supercomputers to solve high-dimensional Bayesian inference problems and produce many posterior samples.

\subsection{Bayesian inference for COVID-19}
Finally, we consider a real-world problem of Bayesian inference for the dynamics of the transmission and severity of COVID-19 using recorded data for New York from \url{https://github.com/COVID19Tracking}, as studied in  \citep{ChenGhattas20,ChenWuGhattas20}. We use a compartmental model for epidemics and the number of hospitalized cases as the observation data to infer the social distancing parameter of 96 dimensions with transformed Gaussian prior. More details on the setup for the model, parameter, and data can be found in \cite{ChenGhattas20}. We run WGD and pWGD using 128 samples with 8 samples in each of 16 processor cores. We update the projection bases for pWGD every 10 of 200 iterations. The eigenvalues have very fast decay as shown in Figure \ref{fig:covid-19}, indicating an intrinsic low dimensionality of the data-informed parameter subspace. We can also see from the middle figure that pWGD produces posterior samples that recover the data better than those of WGD with 90\% credible interval, which is due to the collapse of the WGD samples (especially before April) as can be seen from the bottom figure.

\begin{figure}[!htbp]
    \centering
    \includegraphics[width=0.6\linewidth]{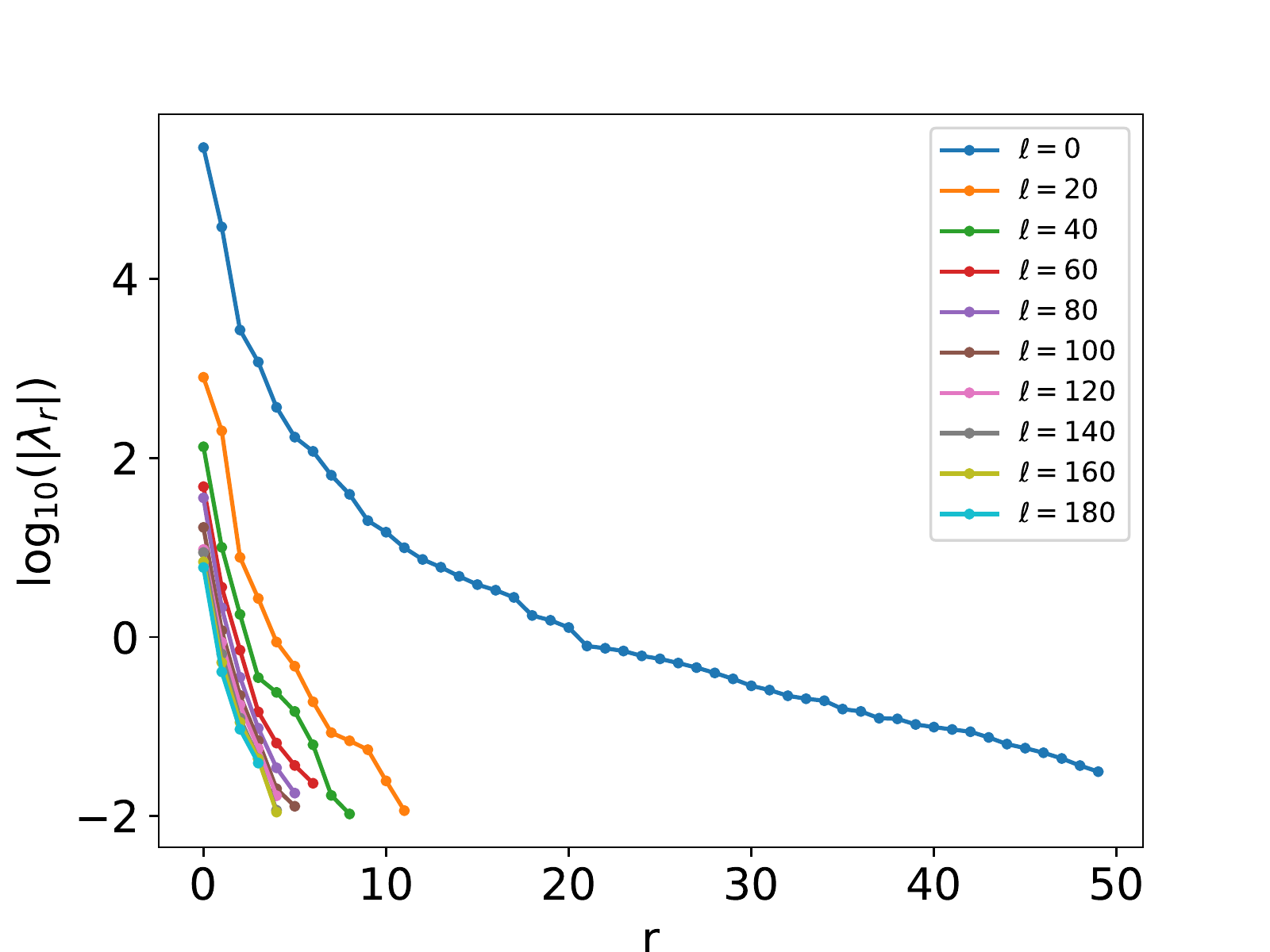}
    \includegraphics[width=0.6\linewidth]{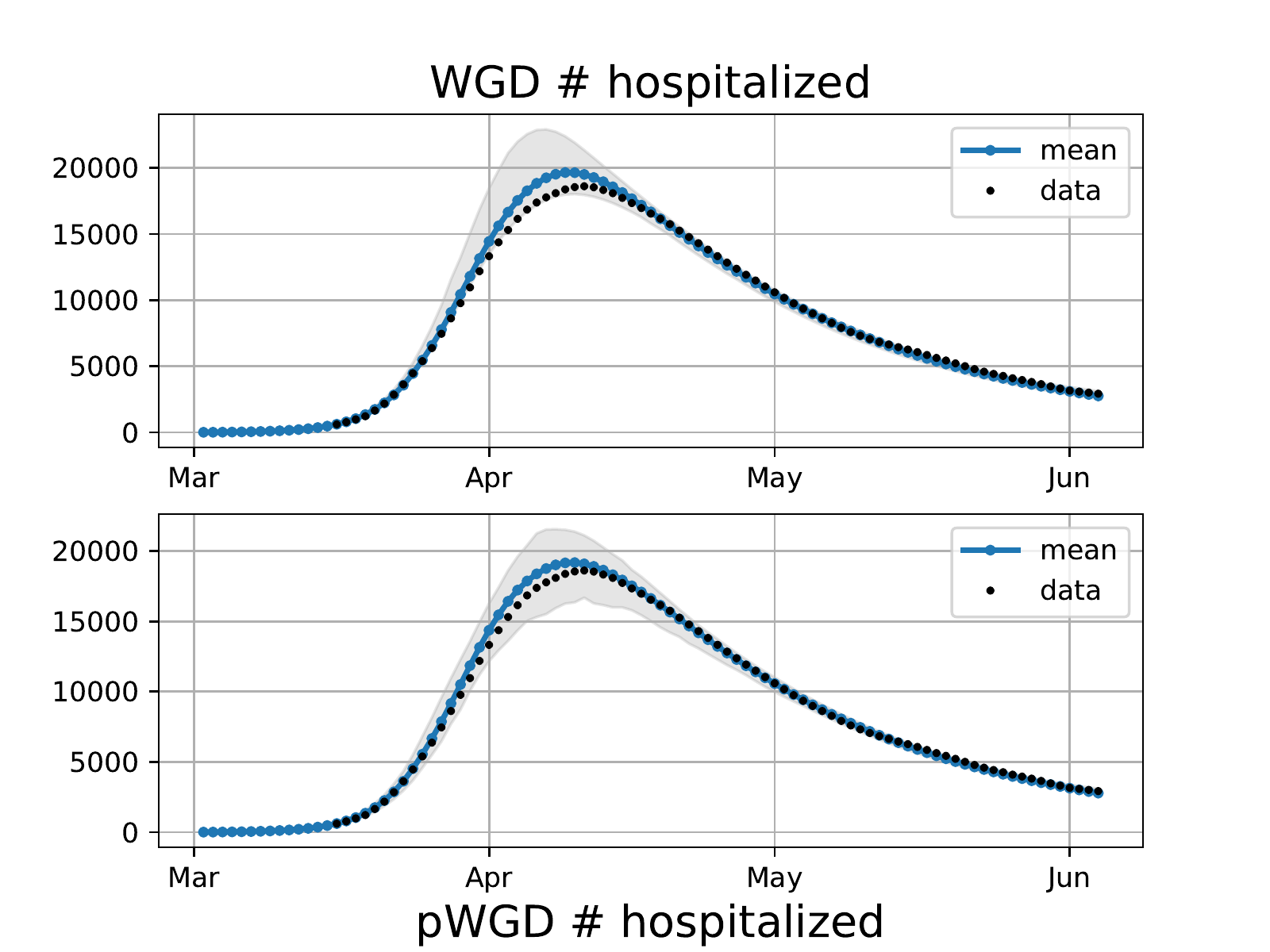}
    \includegraphics[width=0.6\linewidth]{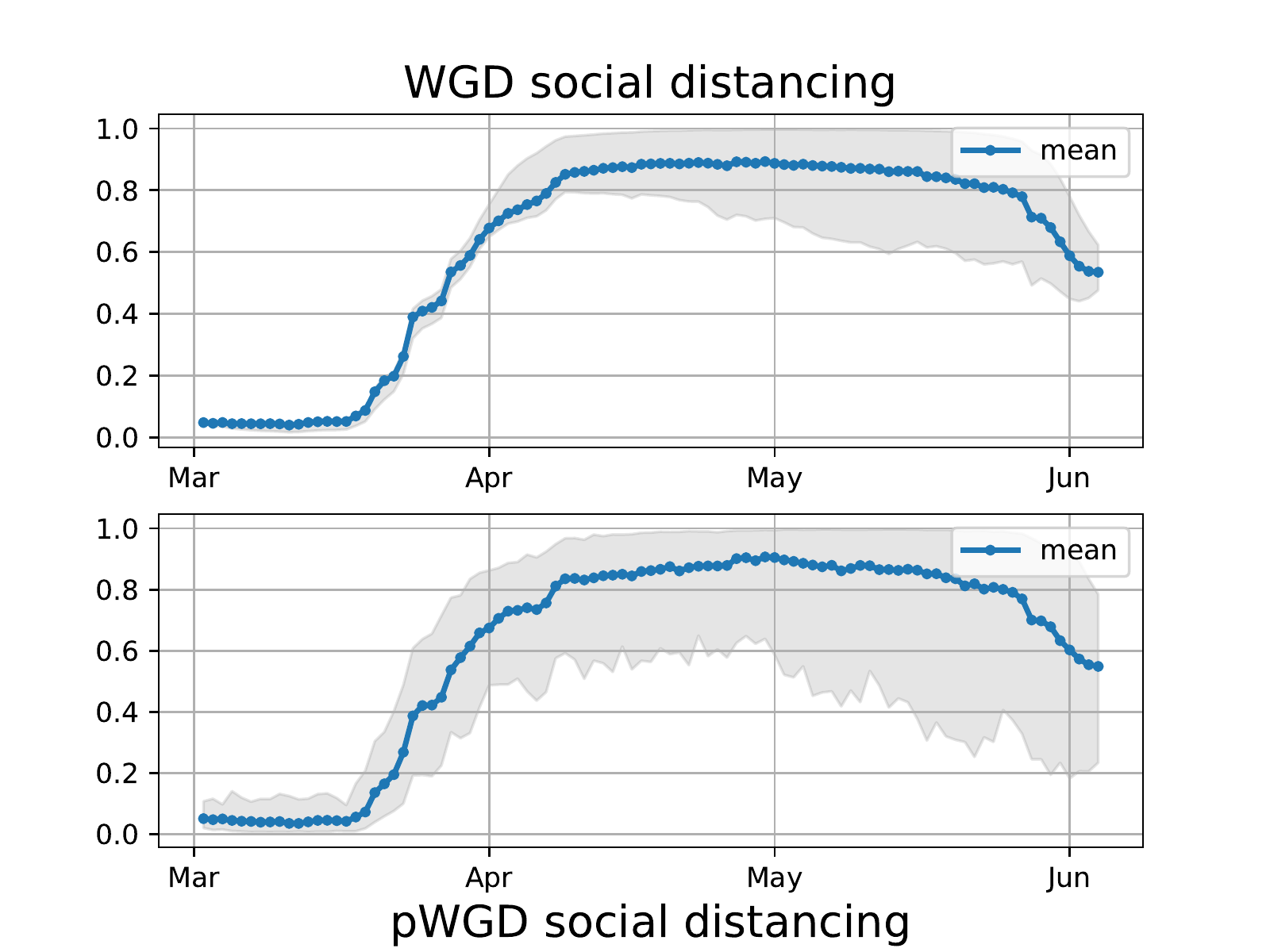}
    \caption{\small Comparison of WGD and pWGD for Covid-19. Top: decay of eigenvlaues at different iterations. Posterior samples (bottom) that predict data (middle) with 90\% credible interval.}
    \label{fig:covid-19}
\end{figure}

\section{Discussion}
In this paper, we develop a pWGD method for high-dimensional Bayesian inference, which effectively alleviates the curse of dimensionality faced in using KDE for approximating the sample density. We analyze the convergence property of pWGD through the lens of Wasserstein gradient flow under suitable assumptions. 
Moreover, we demonstrate the accuracy and convergence (compared to WGD, SVGD, and pSVGD), as well as scalability of the complexity (w.r.t.\ parameter dimension, sample size, processor cores) of pWGD by a variety of experiments. 
Further analyses for the convergence of pWGD and its application to other high-dimensional problems are of great interest. 

\section*{Acknowledgement}
Wang is supported by a department fellowship from the Department of Electrical Engineering in Stanford University. Chen was partially funded by the National Science Foundation, Division of Mathematical Sciences under award DMS-2012453; the Department of Energy, Office of Science, Office of Advanced Scientific Computing Research, Mathematical Multifaceted Integrated Capability Centers (MMICCS) program under award DE-SC0019303; and the Simons Foundation under award 560651. Li is supported by a start-up funding in Univeristy of South Carolina.

\bibliography{WIG,references}
\bibliographystyle{apalike}
\end{document}